\documentclass{article} % DO NOT CHANGE THIS

\PassOptionsToPackage{sort, numbers, compress}{natbib}
\usepackage{arxiv}  
\usepackage{natbib}

\usepackage[utf8]{inputenc} % allow utf-8 input
\usepackage[T1]{fontenc}    % use 8-bit T1 fonts
\usepackage[colorlinks=true, allcolors=black, pdftitle={KAFT}]{hyperref}       % hyperlinks
\usepackage{url}            % simple URL typesetting
\usepackage{booktabs}       % professional-quality tables
\usepackage{amsfonts}       % blackboard math symbols
\usepackage{nicefrac}       % compact symbols for 1/2, etc.
\usepackage{microtype}      % microtypography
\usepackage{xcolor}         % colors

\usepackage{graphicx} % Required for inserting images

\usepackage{newfloat}
\usepackage{listings}

\usepackage{algorithmic, algorithm}
\usepackage{amsmath}
\usepackage{amssymb}
\usepackage{amsthm}
\usepackage{tabularray}
\usepackage{circledsteps}
\UseTblrLibrary{booktabs}

\theoremstyle{plain}
\newtheorem{theorem}{Theorem}

\newtheorem{myquestion}{Question}
\theoremstyle{definition}
\newtheorem{definition}{Definition}

\theoremstyle{remark}

\usepackage[textsize=tiny]{todonotes}

\newcommand\numberthis{\addtocounter{equation}{1}\tag{\theequation}}
\newcommand{\bsTwo}{\boldsymbol{S_2}}
\newcommand{\bbP}{\mathbb{P}}
\newcommand\kear{\textsc{kear}}
\newcommand\raft{\textsc{raft}}
\newcommand\kaft{\textsc{kaft}}
\newcommand\kaftc{\textsc{kaft-c}}
\newcommand\kaftcAbs{\textsc{$\vert$kaft-c$\vert$}}
\newcommand\raftc{\textsc{raft-c}}
\newcommand\raftcAbs{\textsc{$\vert$raft-c$\vert$}}
\newcommand\roar{\textsc{roar}}
\newcommand\roarAbs{\textsc{$\vert$roar$\vert$}}

\DeclareRobustCommand{\flatdownarrow}{
    \begin{tikzpicture}[baseline={(0,-0.25)}]
        \draw[line width=0.2mm, ->] (0,0) arc[start angle=90, end angle=0, radius=0.05] -- (0.05, -0.05) -- (0.05,-0.3);
    \end{tikzpicture}
}

\DeclareRobustCommand{\flatuparrow}{
    \begin{tikzpicture}[baseline={(0,0.05)}]
        \draw[line width=0.2mm, ->] (0,0) arc[start angle=-90, end angle=0, radius=0.05] -- (0.05, 0.05) -- (0.05,0.3);
    \end{tikzpicture}
}

\title{Rethinking Explanation Evaluation under the Retraining Scheme}

\author{
Yi~Cai \\
Dept. of Math. and Comp. Science \\
Freie Universität Berlin\\
Berlin, Germany \\
\texttt{yi.cai@fu-berlin.de} \\
\And
Thibaud~Ardoin \\
Dept. of Math. and Comp. Science \\
Freie Universität Berlin\\
Berlin, Germany \\
\texttt{thibaud.ardoin@fu-berlin.de} \\
\And
Mayank~Gulati \\
Dept. of Math. and Comp. Science \\
Freie Universität Berlin \\
Berlin, Germany \\
\texttt{mayank.gulati@fu-berlin.de} \\
\And
Gerhard~Wunder\\
Dept. of Math. and Comp. Science \\
Freie Universität Berlin\\
Berlin, Germany \\
\texttt{g.wunder@fu-berlin.de}
}
\begin{document}

\maketitle

\begin{abstract}
Feature attribution has gained prominence as a tool for explaining model decisions, yet evaluating explanation quality remains challenging due to the absence of ground-truth explanations.
To circumvent this, explanation-guided input manipulation has emerged as an indirect evaluation strategy, measuring explanation effectiveness through the impact of input modifications on model outcomes during inference.
Despite the widespread use, a major concern with inference-based schemes is the distribution shift caused by such manipulations, which undermines the reliability of their assessments.
The retraining-based scheme \roar{} overcomes this issue by adapting the model to the altered data distribution.
However, its evaluation results often contradict the theoretical foundations of widely accepted explainers.
This work investigates this misalignment between empirical observations and theoretical expectations. 
% NOTE: shorten "in particular"?
In particular, we identify the \textit{Sign} issue as a key factor responsible for residual information that ultimately distorts retraining-based evaluation. 
% NOTE: shorten "identified"?
Based on the analysis, we show that a straightforward reframing of the evaluation process can effectively resolve the identified issue.
Building on the existing framework, we further propose novel variants that jointly structure a comprehensive perspective on explanation evaluation. 
These variants largely improve evaluation efficiency over the standard retraining protocol, thereby enhancing practical applicability for explainer selection and benchmarking.
% OLD --> In this work, we identify the \textit{Sign} issue as a key factor distorting evaluation results in retraining-based schemes and discuss how a simple reframing of the evaluation process can improve assessment validity. 
% OLD --> Beyond analyzing the misalignment between empirical observations and theoretical expectations, we propose efficient variants that further address the high evaluation costs — another major limitation of retraining schemes — thus enhancing their practical applicability for explainer selection and benchmarking.
Following our proposed schemes, empirical results across various data scales provide deeper insights into the performance of carefully selected explainers, revealing open challenges and future directions in explainability research.
\end{abstract}

\section{Introduction}\label{sec:intro}
Explainable AI (XAI) has gained significant attention over the past decade, motivated by the blooming real-world applications of data-driven models.
These models learn decision rules implicitly from data; however, the self-learned decision processes --- shaped by complex and non-linear model architectures --- are often opaque, raising concerns about transparency and trustworthiness.
As a foundational step towards model explainability, feature attribution aims at quantifying the contribution of input features to inquired predictions, thereby identifying the most influential evidence during inference.
Encouraged by the growing demand for transparency, numerous attribution methods have been proposed, each claiming improved performance and distinct advantages.
Despite this progress, an objective assessment of explanation quality has long been a challenge in XAI.
Unlike standard machine learning tasks, feature attribution naturally lacks ground truth that exactly describes model behavior~\citep{hedstrommeta}. 
Moreover, the premise that humans have limited insight to model internals prohibits the application of human judgment in explanation evaluation.
This circumstance creates a \textbf{paradox of explanation evaluation}: assessing explanation quality requires precise knowledge of model behavior, yet explainability is pursued because that behavior is unknown.
In practice, the optimal explainer choice and its hyperparameter configuration vary across use cases, affected by factors such as input modality, model architecture, and feature space dimensionality.
The challenge in explanation evaluation enhances the difficulty in selecting and configuring explainers for a given scenario.

Compromising with the explanation evaluation paradox, input manipulation has become a commonly adopted strategy for indirectly evaluating explanations by quantifying their effectiveness in affecting model outcomes.
% A common approach involves sequentially removing input features according to their attribution scores, then collecting and summarizing model outcomes during the manipulation process~\citep{samek2016evaluating}.
% More effective explainers are expected to trigger greater prediction changes when highly-attributed features are removed.
Intuitively, an effective explainer should identify influential features whose absence triggers significant prediction changes.
We refer to such evaluation strategies as \textbf{inference schemes}, since they investigate explanation impacts at inference time.
While inference schemes offer an efficient means of explanation evaluation, concerns have been raised regarding the reliability of resulting assessments.
\citet{hooker2019benchmark} point out that the performance changes may stem from distribution shifts caused by input manipulations, rather than the absence of influential features.
\citet{wang2024benchmarking} further emphasize this issue by showing that the evaluation outcome of inference schemes can be treated as an optimization objective, delivering manipulation sequences that maximize prediction change but deviate from the intent of explainability --- to reflect model behaviors rather than to optimize the manipulation process.

To address distribution shift during input manipulation, \citet{hooker2019benchmark} proposed \roar{}, arguing that retraining the tested model is mandatory after feature occlusion.
Specifically, the \textbf{retraining scheme} creates manipulated copies of the dataset by removing the most influential features, then retrains the model on these modified copies.
The performance degradation of the retrained model is interpreted as an indicator of explanation effectiveness.
Surprisingly, \roar{} reports that many popular explanation methods perform no better than random~\citep{hooker2019benchmark}.
This striking conflict between empirical observations and the well-founded theoretical backgrounds of feature attribution methods~\citep{sundararajan2017axiomatic, sundararajan2020many, erion2021improving} motivates the investigation presented in this work.
The contributions of this paper are: 
1. We investigate the catastrophic explainer performance reported by \roar{} and identify an implicit, unrealistic assumption as the primary source of evaluation distortion;
2. We further strengthen this claim by highlighting the \textit{Sign} issue as a concrete violation of \roar{}'s assumption;
3. We address the distortion by reframing the \textit{retraining scheme} and propose additional evaluation variants for computational efficiency, promoting their practical applicability;
4. Our empirical results reveal limitations in certain attribution methods, exposing open challenges and future research directions in explainability research.
% \begin{enumerate}
%     \item We investigate the catastrophic explainer performance under the \roar{} scheme and identify an implicit, unrealistic assumption as the primary source of evaluation distortion;
%     \item We further strengthen our argument by identifying the \textit{Sign} issue as a concrete case that violates \roar{}'s assumption;
%     \item We address the distortion by reframing the \textit{retraining scheme} and propose additional evaluation variants for computational efficiency, promoting their practical applicability;
%     \item Our empirical results reveal limitations in certain feature attribution methods, exposing open challenges and future research directions for explainability research.
% \end{enumerate}

\section{Related Work}\label{sec:rw}

The use of inference schemes for explanation evaluation spans from the early stages of XAI studies~\citep{samek2016evaluating, montavon2018methods} to recent developments~\citep{cai2024gradient, muzellec2024saliency}.
The core idea is to assess explanation quality indirectly by examining whether the attribution scores assigned to input features correspond to their actual impact on model predictions~\citep{zeiler2014visualizing, bach2015pixel}.
For explainers that correctly capture relevant evidence, removing the highly attributed features should lead to substantial drops in model confidence.
\citet{samek2016evaluating} formalized this idea by recursively removing features in descending order of attributions and quantifying explanation quality with the area over the perturbation curve (AOPC).
Subsequent studies extended this approach by altering the feature removal order~\citep{petsiuk2018rise, brocki2023feature} and normalizing AOPC scores to mitigate the sensitivity of the measure to the original prediction confidence~\citep{cai2025gefa}.
Similarly, \citet{bhatt2020evaluating} and \citet{yeh2019fidelity} evaluated explainer performance by measuring the correlation between attribution scores and prediction changes under random feature removal.

Although inference schemes are widely adopted for their efficiency, there has been criticism about the validity of their evaluation results due to the out-of-distribution (OOD) concern~\citep{hooker2019benchmark, jain2022missingness}.
The difficulty in disentangling the effects of feature removal (intentional) from the consequences of distribution shift (unintentional) undermines the trustworthiness of inference-based evaluation results.
% \citet{blucher2022preddiff} and \citet{bluecher2024decoupling} try to address this issue under the inference-based protocol by refining the manipulation process.
% While these options indeed reduce the risk of introducing artifacts, they also create a misalignment between the replacement values and the baseline values used during explanation.
% Many attribution methods rely on a predefined baseline to model feature absence, delivering an explanation conditioned on the model, the inquiry, and the baseline.
% A different choice of replacement values during evaluation can introduce biases into the evaluation results.
To address this, \citet{hooker2019benchmark} proposed incorporating model retraining after input manipulation as part of the evaluation process, introducing \underline{r}em\underline{o}ve \underline{a}nd \underline{r}etrain (\roar{}).
\roar{} retrains the target model on manipulated inputs and interprets the retraining accuracy as a proxy for explanation quality.
While retraining mitigates the OOD issue, the assessments by \roar{} contradict theoretical expectations and may appear misleading~\citep{lundberg2020local}. 

\citet{rong2022consistent} and \citet{parkgeometric} sought to improve the retraining protocol by refining the choice of replacement values during feature removal. 
However, since most feature attribution methods are baseline-dependent~\citep{lundberg2017unified,sundararajan2017axiomatic}, decoupling the removal process from the baseline value --- which provides the necessary context to interpret the derived explanations --- can introduce evaluation biases.
In particular, the use of generative models for feature removal \citep{parkgeometric} loses precise control over the manipulation process, challenging the transparency of the evaluation framework.
Beyond concerns of bias, the intensive computational overhead associated with exhaustive retraining remains unaddressed, limiting the broader impact of retraining-based evaluation schemes.

% However, since humans have limited insights into model behaviors, we regard human studies primarily as tools for examining the accessibility and comprehensibility of explanations from a human-machine interaction perspective.
% As such, we exclude human-grounded evaluations from the scope of this work and focus instead on automated evaluation schemes.

\section{Distortion in Retraining Scheme}\label{sec:motivation}
To investigate the source of distortion in retraining-based evaluation, we focus on \roar{} --- the standard protocol in this category. 
Under the \roar{} scheme, an effective explainer is expected to induce greater performance degradation after retraining on manipulated data with the highest attributed features removed~\citep{hooker2019benchmark}.
This expectation implicitly poses the following evaluation question:
\begin{myquestion}[\roar{}]\label{rq_roar}
  Does the tested explainer identify all task-relevant features?
\end{myquestion}

\begin{figure*}
    \centering
    \includegraphics[width=1.\textwidth]{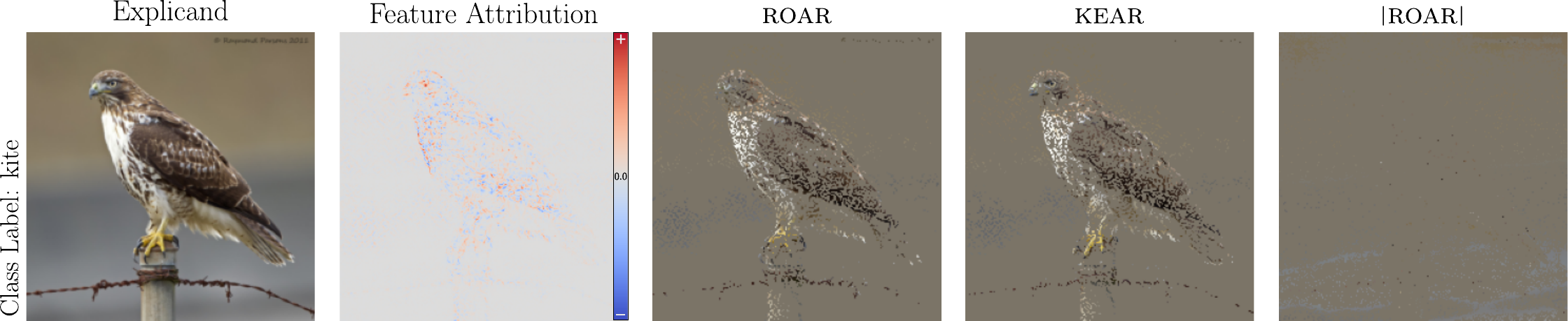}
    % \caption{Example of the \textit{Sign} issue. Left: Original input; Middle: Feature attribution visualized through a saliency map; Right: Manipulated input after highest-first occlusion. As visualized by the manipulated input, the target-overlapping negativities leak information about the target object, contributing to evaluation distortion.}\label{fig:sign_example}
    \caption{Example of the \textit{Sign} issue. 
    The first two columns show the original input and the corresponding feature attributions derived by IG. The subsequent columns present the manipulated inputs after removing $90\%$ of features following three evaluation schemes: \roar{}, \kear{}, and \roarAbs{}.
    For \roar{}, the target-overlapping negative features leak information about the target object, leading to evaluation distortion.
    Despite the visual similarity of the manipulated inputs, retraining following \kear{} achieves over $20\%$ higher accuracy, implying that the explainer correctly identifies contributing features with an appropriate class association.
    As a control group, removing all relevant features (by \roarAbs{}) achieves the expectation described in Equation~\eqref{eq:roar_exp} and yields the lowest accuracy among the three, which is $6\%$ lower than random feature removal.
    }\label{fig:sign_example}
\end{figure*}

\subsection{Distortion by Residual Information}
Let\footnote{Throughout the paper, vectors and sets are typeset in boldface, whereas scalars are presented in plain font.} $\boldsymbol{N}=\{x_1,\ldots,x_n\}$ be the full feature set and $\boldsymbol{S}\subset \boldsymbol{N}$ represent the subset of the most important features for the model function $\boldsymbol{f}(\boldsymbol{N})$.
Borrowing the notion of mutual information from information theory, let $I(\boldsymbol{N};y)$ represent the utility of the input features for predicting the target label $y$. 
The core expectation of Question~\ref{rq_roar} can be formalized as:
\begin{equation}\label{eq:roar_exp}
    I(\boldsymbol{N};y)\gg 0 \overset{\textrm{manipulate}}{\implies} \tilde{I}(\boldsymbol{N}\backslash \boldsymbol{S};y) \approx 0
\end{equation}
where $\tilde{I}$ denotes the remaining utility after distribution shift caused by explanation-guided input manipulation.
While this expectation appears reasonable --- the most important features $\boldsymbol{S}$ should be identified and removed --- a subtle gap between model knowledge and data information can lead to deviations from the expected manipulation results.

Specifically, feature attribution reflects model knowledge in solving a specific task, which does not necessarily align with the true utility of input features as represented by the underlying data distribution.
This misalignment can arise because standard machine learning theory does not guarantee that a trained model will capture and exploit \textit{all relevant features} in accordance with their ground-truth importance.
As a result, residual information can persist in the manipulated data and distort evaluation outcomes, reflected as:
\begin{equation}\label{eq:unexpected}
    \tilde{I}(\boldsymbol{N}\backslash \boldsymbol{S};y)\gg 0
\end{equation}
Retraining after manipulation fits the model to the residual information, producing unexpectedly high accuracies and leading to an underestimation of explanation quality.
Practically, residual information can arise from various sources, such as model underfitting, which overlooks relevant features, or feature redundancy, which renders some informative features completely unused. 
While some factors can be controlled through careful experimental setup, we particularly highlight the \textit{Sign} issue, which can provably harm evaluation validity under mild conditions.

% \begin{figure}
%   \centering  
%   \includegraphics[width=0.35\textwidth]{imgs/feature_subsets.pdf}
%   \caption{By expecting performance degradation, \roar{} implicitly assumes the occlusion of all task-relevant features $\boldsymbol{S}_{\mathcal{D}}$ from the data perspective.
%   In contrast, feature attribution aims to identify the influential features $\boldsymbol{S}_{\mathcal{M}}$ for model decisions.
%   Residual task-relevant information $\boldsymbol{S}_{\mathcal{D}\backslash\mathcal{M}}$ remains after manipulation if $\boldsymbol{S}_{\mathcal{M}}$ does not fully overlap with $\boldsymbol{S}_{\mathcal{D}}$.
%   }\label{fig:mismatch}
% \end{figure}

\subsection{The \textit{Sign} Issue}
Under the highest-first occlusion strategy, features with negative attributions survive the manipulation process.
Despite their negative scores, features with large attribution magnitudes are often highly task-relevant.
During retraining, such ``\textit{negative}'' features can leak substantial task-relevant information, leading to unexpectedly high performance of retrained models. 
We refer to the distortion caused by negatively attributed features as the ``\textit{Sign}'' issue.

In contrast to positively attributed features that support the target decision $y$, negative attributions arise when features serve as evidence for an alternative class $y^*\neq y$ while being shared with $y$.
We refer to such shared features as \textit{secondary evidence} for $y$, formally defined below.
\begin{definition}[Secondary Evidence]\label{def:secondaryEvidence}
    Two disjoint feature subsets $\boldsymbol{S_1}, \boldsymbol{S_2} \subseteq \boldsymbol{N}$ are called the \textit{primary} and \textit{secondary} evidence, respectively, for class $y$ if they satisfy:
    \[
        I(\boldsymbol{S_1};y)>I(\boldsymbol{S_2};y)\gg 0 \textrm{~~and~~} I(\boldsymbol{S_2};y^*) > I(\boldsymbol{S_2};y)
    \] 
\end{definition}
The second condition indicates the greater utility of $\boldsymbol{S_2}$ for $y^*$, which can lead to negative attributions w.r.t. $y$.
When removing features in descending attribution order, these negatively scored secondary features will remain and turn into primary evidence for predicting $y$, as formalized in Theorem~\ref{thm:swap} (see Appendix~\ref{apx:proof_swap} for the detailed proof).
\begin{theorem}[Increasing Utility of Secondary Features]\label{thm:swap}
    The mutual information between $\boldsymbol{S_2}$ and $y$ increases due to distribution shift after input manipulation:
    \[
        \tilde{I}(\boldsymbol{S_2};y) > I(\boldsymbol{S_2};y) \gg 0
    \]
\end{theorem}
The increase in utility pinpoints the emergence of residual information, as described in Equation~\eqref{eq:unexpected}.
Notably, the presence of shared features across multiple classes is common in classification tasks, where the targets often involve overlapping low-level concepts.
The statement in Theorem~\ref{thm:swap} is related to the \textit{information swapping} issue raised by \citet{lundberg2020local} in the context of binary classification; our discussion formalizes and generalizes this concern to a broader multi-class setting.
% A counterintuitive phenomenon in feature attribution is that negatively attributed features often overlap with regions corresponding to the target object.
% Figure~\ref{fig:sign_example} qualitatively illustrates this overlap and shows how the highest-first occlusion strategy leaks task-relevant information due to the \textit{Sign} issue. 
% While some target-overlapping negatives may reflect explainer's errors and warrant penalties in quality assessments, they are not inherently indications of poor explanations.
% % One may ask: ``could the unexpected negativities be mistakes made by explainers, which should be reflected in final quality assessments?''
% % Although we should not rule out the possibility that some target-overlapping negativities are mistakes by explainers, these negative features are not necessarily indications of bad explanations.
% To complement this argument, we introduce the concept of weak positive contributors, which concretizes the case where target-related features can receive negative attributions due to specific model reasoning patterns.

\begin{figure*}
    \centering
    \includegraphics[width=1.\textwidth]{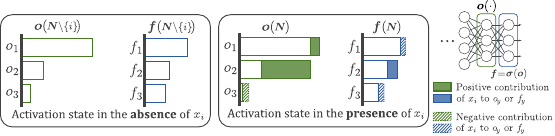}
    \caption{Example of a weak positive contributor. The feature $x_i$ is a weak contributor to class $1$: it positively influences $o_1$ but yields a negative contribution to $f_1$.}\label{fig:softmax_example}
\end{figure*}
\subsection{Weak Positive Contributor}
Figure~\ref{fig:sign_example} qualitatively illustrates information leakage caused by the \textit{Sign} issue under the highest-first occlusion strategy.
Although Theorem~\ref{thm:swap} establishes the utility of negatively attributed features, it may still appear counterintuitive when regions representing the target object receive negative attribution scores.
This subsection concretizes a case where even positively correlated features can receive negative attributions, thereby demonstrating that the target-overlapping negatives are not necessarily indicative of poor explanation quality.
% which allows a probabilistic interpretation of the prediction outcome.

Let $\boldsymbol{o}(\boldsymbol{N})$ be the output of the final dense layer in a classifier (i.e., the classification head) and $o_y(\boldsymbol{N})$ be the activation corresponding to class $y$. 
A common practice in classification tasks is to apply a softmax function to normalize the raw model outcomes at inference time: $\boldsymbol{f}(\boldsymbol{N})=\boldsymbol{\sigma}(\boldsymbol{o}(\boldsymbol{N}))$.
However, this softmax normalization can unintentionally result in target-overlapping negative attributions.
In certain cases, even features that positively activate the output node $o_y(\boldsymbol{N})$ can dramatically receive negative attributions for the final prediction $f_y(\boldsymbol{N})$.
We refer to such features as \textit{weak positive contributors}.
\begin{definition}[Weak Positive Contributor]\label{def:wpc} 
    A feature $x_i$ is called a \textit{weak positive contributor} to class $y$ if its attribution $\xi^{o_y}_{i}$ to the output logit $o_y$ satisfies:
    \begin{equation}\label{eq:wpc}
        \rlap{$\underbrace{\phantom{1 < \exp(\xi^{o_y}_{i})}}_{\textrm{Positive Contribution}}$}
        1 < \overbrace{\exp(\xi^{o_y}_{i}) < \mathbb{E}_{y^*\neq y}\Bigl[\exp(\xi_i^{o_{y^*}})\Bigr]}^{\textrm{Weak Contribution}}
    \end{equation}
\end{definition}
By writing the expectation $\mathbb{E}_{y^*\neq y}$, we interpret the final model output $\boldsymbol{f}(\boldsymbol{N})$, normalized by a softmax layer, as a probability distribution over labels conditioned on $y^*\neq y$.
The first inequality in \eqref{eq:wpc} indicates the positive contribution to $o_y$, whereas the second implies that this contribution is weak relative to its expected exponential contributions to the remaining classes.
Please note that $\xi_i^{o_y}$ refers to the ground-truth attribution, i.e., the quantity that feature attribution methods seek to estimate. 
Therefore, the discussion of weak positive contributors is independent of any specific explanation method and reflects intrinsic properties of the model's predictive behavior.

\begin{theorem}[Negative Attribution]\label{thm:wpc}
  The attribution $\xi_{i}^{f_y}$ of a weak positive contributor w.r.t. the final prediction is negative, despite its positive attribution to $o_y$.
\end{theorem}
Figure~\ref{fig:softmax_example} showcases that a positively associated feature can negatively affect the prediction confidence of the corresponding class, as stated in Theorem~\ref{thm:wpc} (see Appendix~\ref{apx:proof_wpc} for proof and further details).

% TBD: Here we use the example of weak positive contributors as a qualitative example showcasing the sign issue, but please note that the target-overlapping negativities can happen in a more general context.
% for some low-level features
% to achieve better generalizability and balance the performance across different classes

\section{Keep and Fine-tune (\kaft{})}\label{sec:kaft}
Based on the previous discussion, a direct remedy for the \textit{Sign} issue is to perform feature manipulation in descending order of attribution magnitudes.
This variant is denoted by \roarAbs{}, where the absolute value operator implies that only attribution magnitudes are used to rank features.
We refer to this occlusion strategy as \textit{relevant}-first occlusion, since high-magnitude attributions reflect feature relevance, regardless of their signs.
However, \roarAbs{} remains sensitive to residual information arising from redundant features, which are informative but receive low attribution scores with their significance masked by redundancies.

An alternative is to reframe the occlusion strategy completely, thereby relieving the evaluation scheme from the over-restrictive expectation of $\tilde{I}(\boldsymbol{N}\backslash\boldsymbol{S};y)\approx 0$.
Instead of removing the most important features, we argue that retaining the top-ranked features for retraining yields more reliable assessments of explanation quality.
\underline{Ke}ep \underline{a}nd \underline{r}etrain (\kear{}) reframes the evaluation question as follows:
\begin{myquestion}[\kear{}]\label{rq_kear}
  Does the tested explanation method correctly identify influential features for model decisions?
\end{myquestion}
The lowest-first occlusion strategy was first mentioned by \citet{hooker2019benchmark}. Appendix~\ref{apx:extended_rw} further discusses the role of \kear{} in prior work and interprets the differences in empirical observations.
By reversing the occlusion priority, an effective explanation method should capture relevant information learned by the target model, resulting in less performance degradation after retraining with the same proportion of retained features. 
Compared to \roar{}, Question~\ref{rq_kear} loosens the assumption of exact model-data alignment and shifts the focus towards model behavior.

% Besides, as the previous discussion on residual information is surrounding the sign issue, another intuitive alternative would be refining the deletion-based manipulation following attribution amplitude, referred to as \textbar{}ROAR\textbar{}.

\begin{figure*}
  \centering  
  \includegraphics[width=1.\textwidth]{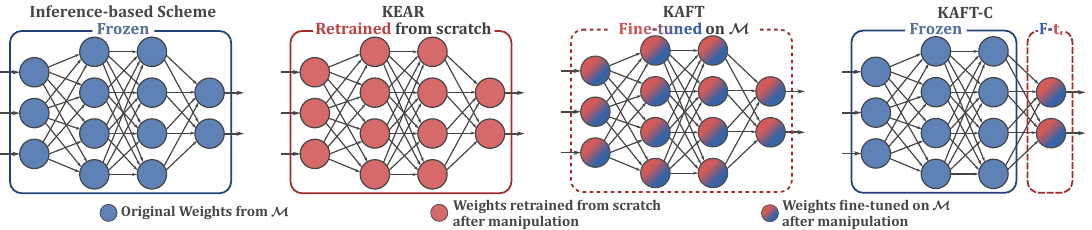}
  \caption{Visualization of model updates under different evaluation schemes.
  The inference scheme is sensitive to distribution shifts as it does not update the tested model.
  \kaft{} and its variants overcome this limitation while mitigating distortions inherent in other retraining schemes.
  Particularly, \kaftc{} evaluates explanation quality by assessing the utility of preserved features through restricted updates, striking a balance between resolving distribution shifts and maintaining focus on model behavior.
  }\label{fig:kear_overview}
\end{figure*}

\subsection{Retraining? Fine-tuning!}
Parallel to distortion in evaluation results, the heavy computational cost also limits the practical use of retraining schemes for explainer selection and benchmarking.
The evaluation costs are twofold: the expense of deriving explanations for numerous training data, and the overhead of the multiple retraining epochs.
Unlike inference schemes, which typically require explanations for only the test set or even a subset of it, retraining schemes must derive explanations for the entire dataset, including training, validation, and test splits, to reproduce the training environment.
Additionally, these manipulated entries are visited repeatedly during retraining to refine model performance.
While such costs are manageable in small-scale settings, efficiency concerns become significant as training pipelines become more complicated.
This challenge is pronounced when benchmarking explainers on downstream models fine-tuned from pretrained versions, where reproducing the entire training process becomes infeasible.

Given that retraining serves to adapt the target model to the disrupted data distribution, we propose replacing full retraining with model fine-tuning.
When the manipulated dataset is viewed as representing a different distribution that preserves partial information overlap with the original data, fine-tuning offers a lightweight alternative for model adaptation under distribution shift.
By leveraging the explained model as initialization, fine-tuning requires fewer training samples and epochs.
More specifically, \underline{k}eep \underline{a}nd \underline{f}ine-\underline{t}uning (\kaft{}) creates manipulated copies of a subset of training samples and, instead of training from scratch, fine-tunes the explained model on the manipulated data.
As in the original retraining scheme, the performance of the fine-tuned model serves as an indicator of explanation quality.
% Transferring the evaluation scheme from retraining to fine-tuning not only reduces evaluation costs but also strengthen the connection to the target model, which is of the interest of feature attribution, rather than the dataset.

Following this direction, we further simplify the evaluation scheme by restricting \underline{k}eep \underline{a}nd \underline{f}ine-\underline{t}uning to the \underline{c}lassification head (\kaftc{})\footnote{Abbreviation naming rules are detailed in Appendix~\ref{apx:exp_details}.}.
This simplification is inspired by the observation that the hidden layers of a model function primarily as feature extractors, deriving and summarizing high-level representations that support the final decisions at the classification head.
An effective explainer should capture input features that contribute to these informative high-level representations (from the model perspective) and ensure that they are preserved during explanation-guided manipulation.
Although input perturbation may degrade model performance by disrupting contextual dependencies, the feature extraction performed by the hidden layers should remain at least partially effective if the most relevant input components are preserved.
Consequently, adapting the classification head allows the model to reuse the extracted representations under the altered context, achieving higher classification accuracy without full model updates.
From the efficiency perspective, when fine-tuning with $10\%$ of the data and $10\%$ of the training epochs used for full retraining, \kaftc{} achieves a $100\times$ speedup over \kear{}.
Figure~\ref{fig:kear_overview} provides an overview of the three retraining alternatives and compares them to the inference-based evaluation scheme.

Compared to full retraining or unrestricted fine-tuning, fine-tuning only the classification head further simplifies the retraining process and strengthens the connection to the explained model.
This modification better aligns the evaluation process with the goal of feature attribution methods.
By freezing the hidden layers of the tested model, the evaluation concentrates on the original behavior of the target model without completely altering its internal characteristics.
% This restricted model update minimizes potential distortions in evaluation results caused by model-data misalignment.
This restriction further mitigates the distortion caused by feature redundancy when applying \textit{relevant-first} occlusion, as unused features are no longer reorganized through full retraining.
Following this direction, we propose another alternative evaluation scheme \raftcAbs{} (\underline{r}emove \underline{a}nd \underline{f}ine-\underline{t}uning on the \underline{c}lassification head following attribution magnitude).
Compared to \kaftc{}, which reflects the utility of the highest-attributed features, \raftcAbs{} assesses the ability of an explainer to identify irrelevant features, providing a complementary perspective for assessing explanation quality.
% Together with \kaftc{}, the two occlusion priorities structure a comprehensive perspective for evaluating explanation quality. While \kaftc{} assesses the explainer's ability in identifying features that do not contribute to the target class, the assessments by \raftcAbs{} reflect the correctness of high-magnitude attributions.

\section{Experiments}
To demonstrate the validity of the proposed evaluation schemes and their differences from \roar{}, we begin with small-scale experiments with a carefully selected set of explainers.
These explainers are chosen based on their different theoretical backgrounds, serving as a reference for explanation quality.
Following the small-scale tests, we adopt \kaftc{} and \raftcAbs{}, the most efficient retraining variants, for benchmarking an extended set of attribution methods in large-scale settings.
% Table~\ref{tab:abbrvs} summarizes the abbreviations and corresponding configurations of the evaluation strategies.

Across all tests, full retraining (\roar{} and \kear{}) is performed on the complete dataset, whereas $20\%$ and $10\%$ of the training samples are used for fine-tuning and classification-head-only fine-tuning, respectively.
All reported values are averages over five repetitions for better reliability of the results.

\subsection{Small-scale Experiments}

\subsubsection{Datasets and Classifiers} 
The small-scale experiments are conducted on three publicly available datasets: MNIST~\citep{lecun1998gradient}, CIFAR10~\citep{krizhevsky2009learning}, and STL10~\citep{coates2011analysis}, covering input types with varying color models and resolutions.
A simple CNN, WideResNet~\citep{zagoruyko2016wide}, and EfficientNet-B0~\citep{tan2019efficientnet} are trained on MNIST, CIFAR10, and STL10, respectively.
All classifiers are trained from scratch for efficient reproduction of the exact training environments during retraining.

\subsubsection{Feature Attribution Methods}
Three gradient-based attribution methods with various theoretical backgrounds are selected: Vanilla Gradient (\textbf{VG})~\citep{simonyan2014deep}, \textsc{SmoothGrad} (\textbf{SG})~\citep{smilkov2017smoothgrad}, and Integrated Gradients (\textbf{IG})~\citep{sundararajan2017axiomatic}.
When configured with identical query budgets, SG and IG share the same complexity. 
However, IG is distinguished by explicitly modeling feature absence with a baseline.
This baseline, which also defines feature absence during input manipulation, anchors the attribution process, allowing more accurate quantification of how feature presence contributes relative to absence.
Founded on its theoretical strength, IG is expected to outperform SG, which in turn should outperform VG by reducing sensitivity to gradient noise with smoothing.
Appendix~\ref{apx:exp_details} provides additional experimental details about the explainers and models, including implementation specifics and configurations for (re-)training and fine-tuning.

\begin{figure*}
  \centering
  \includegraphics[width=1.\textwidth]{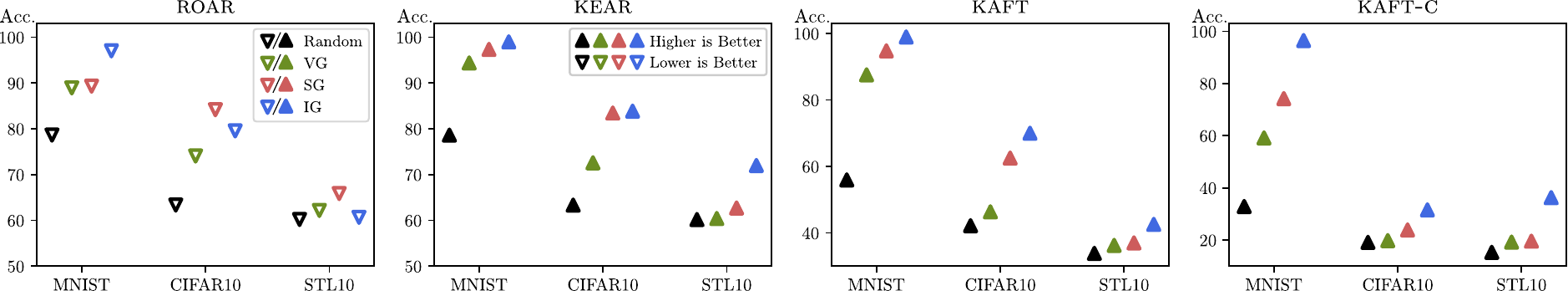}
  \caption{Explainer performance reflected by different schemes with $90\%$ of features removed. Evaluation results of each scheme are grouped by test case.}\label{fig:smallScal90p}
\end{figure*}

\begin{figure*}
  \centering
  \includegraphics[width=1.\textwidth]{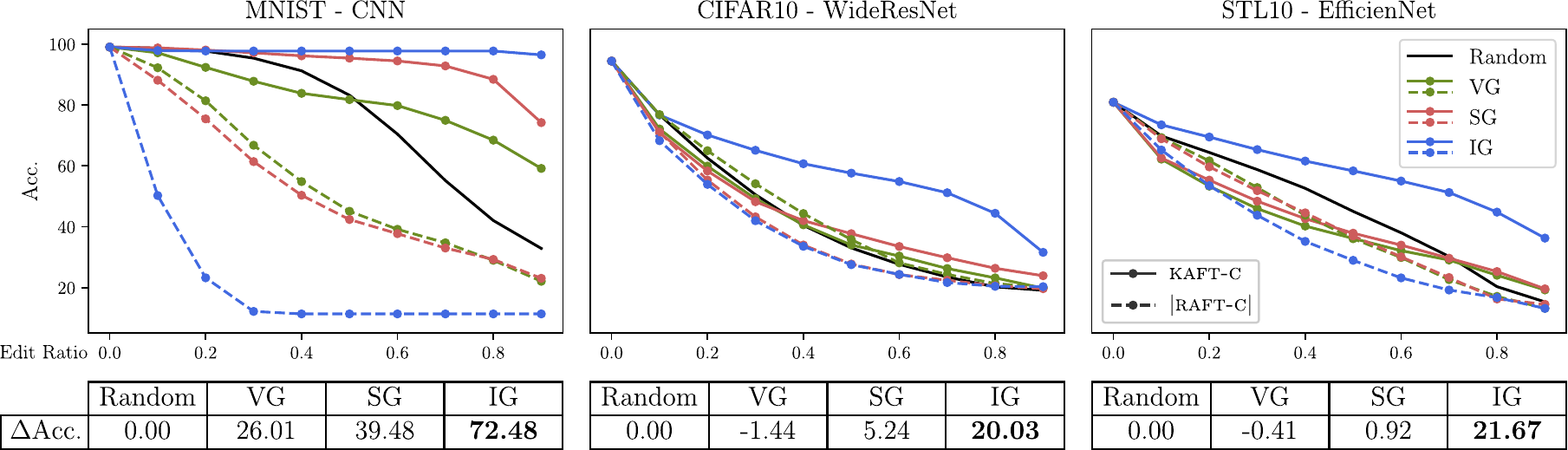}
  \caption{Explanation quality in \textit{small-scale} settings tested by \kaftc{} and \raftcAbs. The embedded tables present $\Delta\textrm{Acc.}$, which quantifies explainer performance.
  $\Delta\textrm{Acc.}$ represents the difference in prediction power, computed as the area between the degradation curves produced by \kaftc{} and \raftcAbs{}.}\label{fig:ratioSmall}
\end{figure*}

\subsubsection{Small-scale Results}
We start with empirically setting the manipulation ratio to $90\%$ and retrain the models on the manipulated datasets. 
Figure~\ref{fig:smallScal90p} illustrates the retraining performances\footnote{Error bars are not presented for visual clarity. Please refer to Appendix~\ref{apx:additional_results} for the complete results and raw numerical values.}.
As an indicator of explanation quality, the interpretation of retraining accuracies depends on deletion priority: for \kear{}, higher accuracy suggests better explanation quality, as it reflects the preservation of influential features; for \roar{}, lower accuracy is preferred, which indicates the effective removal of task-relevant features.

First, we concentrate on the relative ranking of competitors within each evaluation scheme. 
Our results reproduce the findings of~\citet{hooker2019benchmark}, showing that all competitors perform worse than random when tested under the \roar{} scheme, as demonstrated by the limited performance degradation.
As discussed in Section~\ref{sec:motivation}, these misleading assessments arise from residual information, where the remaining negatively attributed features lead to information leakage and, consequently, unexpectedly high retraining accuracies. 
Moreover, the variability in explainer rankings across test cases --- for example, IG ranking highest with STL10-EfficientNet but lowest with MNIST-CNN --- suggests that \roar{}'s assumption (inherited from Question~\ref{rq_roar}) is highly sensitive to the characteristics of tested models.
This sensitivity leads to varying degrees of distortion depending on the test scenario.
In contrast, the results by \kaft{} and other variants of lowest-first occlusion match the theoretically grounded expectation: IG consistently outperforms other explainers due to its explicit use of a baseline.
Regardless of how a model is updated, the \kaft{} family delivers coherent assessments of explanation quality, as reflected by the consistent competitor rankings. 

To further demonstrate the effectiveness of gradient-based approaches and to show that explainers capture not only influential features but also distinguish them from irrelevant ones, we adopt \kaftc{} and \raftcAbs{}, reporting evaluation results across manipulation ratios ranging from $0.1$ to $0.9$ in increments of $0.1$. 
% As an alternative to mitigate the sign issue, \raftcAbs{} removes influential features according to their attribution amplitudes. 
% By limiting model update to only the classification head, \raftcAbs{} also avoids potential distortions through reorganization of residual information as can happen in standard retraining.
Figure~\ref{fig:ratioSmall} shows model performance degradation curves as the manipulation ratio increases. 
The solid and dashed lines correspond to performance degradation under different manipulation priorities.
For an effective explainer, the solid line should lie above the random baseline (black), while the dashed line should fall below.
Among the competitors, only IG exhibits curves that are well-separated by the random baseline across all three test cases, indicating the effectiveness of its explanations.
% , reflecting its ability to distinguish influential features from irrelevant ones.
Complementing the line plots, explainer performance is further quantified by the area between the solid and dashed lines associated with each explainer.
This area reflects the difference in prediction power between irrelevant-first and relevant-first manipulations.
The table embedded in Figure~\ref{fig:ratioSmall} confirms the superiority of IG with the largest $\Delta\textrm{Acc.}$, highlighting its capability in identifying influential features to model decisions. 
Appendix~\ref{apx:additional_results} provides additional qualitative examples that visualize the manipulated inputs, emphasizing the \textit{Sign} issue. 

\begin{figure*}
  \centering
  \includegraphics[width=1.\textwidth]{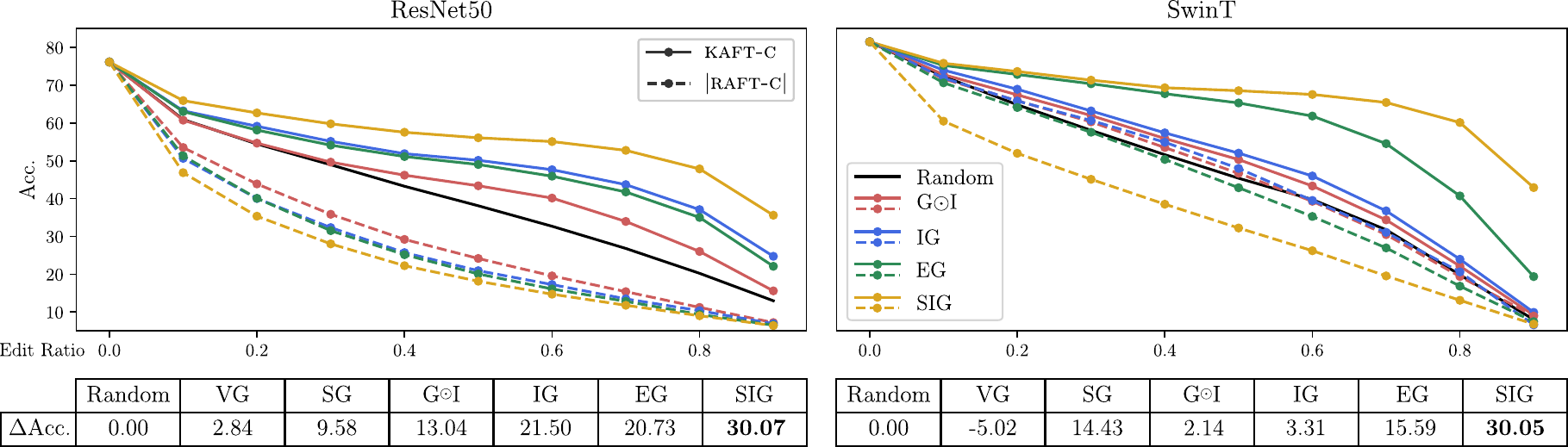}
  \caption{Explanation quality in \textit{large-scale} settings assessed by \kaftc{} and \raftcAbs{}. The plots present a selected subset of competitors for clarity. 
  The tables report $\Delta\textrm{Acc.}$ for all tested explainers.}\label{fig:ratioLarge}
\end{figure*}

\subsection{Large-scale Experiments}\label{sec:exp_large}
\subsubsection{Additional Competitors}
We extended the competitor set with Gradient$\odot$Input (\textbf{G$\odot$I})~\citep{shrikumar2017learning}, Expected Gradients (\textbf{EG})~\citep{erion2021improving}, and Smoothed Integrated Gradients (\textbf{SIG}).
The last approach is an ensemble of IG and SG, combining theoretical grounding with denoising effects.
All selected competitors are widely used and generalizable feature attribution methods.

\subsubsection{Dataset and Classifiers}
We benchmark the competitors on classifiers trained with ImageNet1k~\citep{russakovsky2015imagenet}. 
For architectural diversity, we consider \textit{ResNet50}~\citep{he2016deep}, a CNN-based model, and \textit{SwinT}~\citep{liu2021swin}, a transformer-based model.
For both architectures, pretrained versions are used.

\subsubsection{Large-scale Results}
We repeat the \kaftc{} and \raftcAbs{} evaluations in the large-scale settings, with results presented in Figure~\ref{fig:ratioLarge}.
For ResNet50, attribution methods that explicitly incorporate a baseline generally outperform those that do not.
Notably, G$\odot$I, which can be viewed as a special case of IG with a single observation on the interpolation path, improves explanation quality by amplifying raw gradients with the corresponding input.
In both cases, limited performance degradation is achieved by the best-performing explainer (SIG) under the restrictive fine-tuning setting.
This observation supports the use of fine-tuning as an effective and substantially more efficient alternative to full retraining for handling distribution shifts.

On the other hand, IG exhibits a performance collapse when tested on SwinT.
As shown in the right panel of Figure~\ref{fig:ratioLarge}, the degradation curves produced by IG are nearly indistinguishable from the random baseline in both manipulation orders, indicating a failure to identify features relevant to model decisions.
We consider two possible causes for the failure of IG on SwinT: 1.~a long plateaux of gradient saturation, and 2.~the cancellation effect due to feature interactions.
First, a key motivation of the integration in IG is to resolve gradient saturation~\citep{sundararajan2017axiomatic}.
However, when the interpolation path between an explicand and the chosen baseline traverses a prolonged saturation region, the evenly interpolated instances will be dominated by the saturated gradients.
As a result, the integrated gradients become biased towards the saturation, failing to address the issue.
The similar performance of IG and G$\odot$I on SwinT supports this interpretation --- a long saturation plateau undermines the effectiveness of integration over an even interpolation, reducing IG to a single-sample explanation, as in the case of G$\odot$I.
Second, the straightline integration path can cause underestimation of feature importance due to cancellation effects from feature interactions. 
Consider the simple function $f(x_1,x_2)={(x_1 - x_2)}^2$ as a concrete example.
When explaining the decision at $(1, 1)$ relative to the baseline $(0, 0)$, IG assigns $0$ attribution to both features, as their opposing contributions cancel each other out along the interpolation path, despite both being influential to the outcome.
This limitation can be mitigated by averaging over multiple integration paths that deviate from the straightline, which decomposes interactions into individual contributions.
This simple solution also underlies the improved performance of SIG:
by adding noise to the interpolated points, SIG implicitly averages over observations from diverse paths, thereby revealing interaction effects and recovering the performance of IG.

\section{Conclusion}
This work addresses the gap between the empirical assessments by retraining-based schemes and the theoretical foundations of gradient-based approaches. 
By identifying the cause of evaluation distortion, we propose several alternatives that reframe the evaluation objective to resolve this issue.
Among them, the fine-tuning variant \kaftc{} largely reduces the computational overhead associated with retraining-based evaluation.
Our empirical results generally align with theoretical expectations for the tested explainer --- except for the test on SwinT, where IG exhibits a performance collapse.
We discuss two possible causes of the collapse: saturation plateaux and cancellation effects, both linked to the integration path.
The discussion on the limitation highlights the need for further investigation into path selection and allocation of feature interactions.
Overall, this work provides a new toolset for explanation evaluation, which differs from traditional inference-based schemes, and offers a complementary perspective for benchmarking attribution methods.

\section*{Acknowledgement}
Yi Cai, Thibaud Ardoin, and Gerhard Wunder were supported by the Federal Ministry of Education and Research of Germany (BMBF) in the program of ``Souver\"an. Digital. Vernetzt.'', joint project ``AIgenCY: Chances and Risks of Generative AI in Cybersecurity'', project identification number 16KIS2013. 
Mayank Gulati and Gerhard Wunder were supported by BMBF joint project ``6G-RIC: 6G Research and Innovation Cluster'', project identification number 16KISK020K.

% both accuse the path choice of IG.
% computation overhead to maintain symmetry property of IG
% careful selection of integration path
% and a further investigation in feature interactions.

\bibliographystyle{abbrvnat}
% % \bibliographystyle{ksfh_nat}
\bibliography{ref}

% \input{aaai2026/ReproducibilityChecklist}

%%%%%%%%%%%%%%%%%%%%%%%%%%%%%%%%%%%%%%%%%%%%%%%%%%%%%%%%%%%%

\clearpage
\onecolumn
\appendix

\setcounter{secnumdepth}{2}

\renewcommand\thesubsection{\Alph{section}.\arabic{subsection}}
\renewcommand\labelenumi{\arabic{enumi}}

\section{Technical Appendices and Supplementary Material}

\subsection{Proof of Inreasing Utility of Secondary Evidence}\label{apx:proof_swap}
% TODO: interpretation: for features originally have distinguishing power, the manipulation enhances such power by resolving the mixture of multiple classes
We start with introducing the notations necessary for the proof of Theorem~\ref{thm:swap}.
To facilitate a more tangible analysis, we first restrict the discussion to the case where $S_2$ represents shared features between \textit{only two classes}: $y$ and $y^*$, which gives the premise that:
\begin{equation}\label{eq:shared}
    \bbP{}(y=1,\bsTwo{}=1) + \bbP{}(y^*=1,\bsTwo{}=1) = \bbP{}(\bsTwo{}=1)
\end{equation}
We will discuss how the conclusion generalizes to shared features across more than two classes shortly after completing the proof.
Given a data distribution represented by a balanced dataset with $C$ classes, the mutual information between $S_2$ and $y$ is given by:
\begin{align*}
    &I(\boldsymbol{S_2};y) = H(y) - H(y|\boldsymbol{S_2}) \\
    &H(y) = - \Bigl[\bbP{}(y=1)\cdot\log_C \bbP{}(y=1) + \bbP{}(y=0)\cdot\log_C \bbP{}(y=0)\Bigr] \\
    &\begin{aligned}
        H(y|\bsTwo{}) = - \Bigl[ 
        \bbP{}&(y=1,\bsTwo{}=1)\cdot\log_C\bbP{}(y=1|\bsTwo{}=1) 
        + \bbP{}(y=1,\bsTwo{}=0)\cdot\log_C\bbP{}(y=1|\bsTwo{}=0)  \\
        &+ \bbP{}(y=0,\bsTwo{}=1)\cdot\log_C\bbP{}(y=0|\bsTwo{}=1) 
        + \bbP{}(y=0,\bsTwo{}=0)\cdot\log_C\bbP{}(y=0|\bsTwo{}=0) 
    \Bigr] 
    \end{aligned} \numberthis\label{eq:condEntropy}
\end{align*}
Without introducing \textit{any additional assumption}, we denote the marginal probabilities of $y=1$ and $\bsTwo{}=1$ as:
\begin{align*}
    &\bbP{}(y=1)=\gamma& &\bbP{}(\bsTwo{}=1)=p
\end{align*}
Recalling that $\bsTwo{}$ is shared only between $y$ and $y^*$, we note the co-occurrence of $y=1$ and $\bsTwo{}=1$ as proportional to $\bbP{}(\bsTwo{}=1)$, which yields:
\begin{align*}
    % &
    &\bbP{}(y=1,\bsTwo{}=1) = \alpha p &   
    &\bbP{}(y=1,\bsTwo{}=0) = \gamma - \alpha p  \\
    &\bbP{}(y=0,\bsTwo{}=1) = (1 - \alpha) p &   
    &\bbP{}(y=0,\bsTwo{}=0) = 1 - \gamma - p + \alpha p 
\end{align*}
Applying the definition of conditional probability further derives related expressions that relate to the computation of the conditional entropy $H(y|\bsTwo{})$:
\begin{align*}
    &\bbP{}(y=1|\bsTwo{}=1) = \alpha &   
    &\bbP{}(y=1|\bsTwo{}=0) = \frac{\gamma - \alpha p}{1-p}\\
    &\bbP{}(y=0|\bsTwo{}=1) = (1 - \alpha) &   
    &\bbP{}(y=0|\bsTwo{}=0) = \frac{1 - \gamma - p + \alpha p }{1-p}
\end{align*}

\begin{proof}[Proof of Theorem~\ref{thm:swap}]
    To prove:
    \begin{equation*}
        \boxed{\tilde{I}(\bsTwo{};y) > I(\bsTwo{};y) \gg 0 }
    \end{equation*}
    our \textbf{goal is to show} that the conditional entropy of $y$ given $\bsTwo{}$ is reduced after input manipulation:
    \begin{align*}
        &\tilde{H}(y) - \tilde{H}(y|\bsTwo{}) > H(y) - H(y|\bsTwo{}) \\
        \overset{H(y)=\tilde{H}(y)}{\implies}~& \boxed{H(y|\bsTwo{}) - \tilde{H}(y|\bsTwo{}) > 0} \numberthis\label{eq:goal_swap}
    \end{align*}
    The simplification of \eqref{eq:goal_swap} is justified because feature occlusion does not alter the label distribution. 
    Therefore, the first term on each side cancels out. 
    Similar to \eqref{eq:condEntropy}, the second term on the left-hand side of \eqref{eq:goal_swap} expands as:
    \begin{align*}
        \tilde{H}(y|\bsTwo{}) = - \Bigl[ 
            \tilde{\bbP{}}&(y=1,\bsTwo{}=1) \cdot \log_C\tilde{\bbP{}}(y=1|\bsTwo{}=1) 
            + \tilde{\bbP{}}(y=1,\bsTwo{}=0) \cdot \log_C\tilde{\bbP{}}(y=1|\bsTwo{}=0)  \\
            &+ \tilde{\bbP{}}(y=0,\bsTwo{}=1)\cdot\log_C\tilde{\bbP{}}(y=0|\bsTwo{}=1) 
            + \tilde{\bbP{}}(y=0,\bsTwo{}=0)\cdot\log_C\tilde{\bbP{}}(y=0|\bsTwo{}=0) 
        \Bigr] 
    \end{align*} 
    When applying highest-first occlusion, $\tilde{\bbP{}}(y=1,\bsTwo{}=1)$ and $\tilde{\bbP{}}(y=1,\bsTwo{}=0)$ remain unchanged, as $\bsTwo{}$ represents negatively attributed secondary evidence w.r.t. $y$ and thus survives the manipulation:
    \begin{align*}
        & \tilde{\bbP{}}(y=1,\bsTwo{}=1) = \bbP{}(y=1,\bsTwo{}=1) = \alpha p 
        & \tilde{\bbP{}}(y=1,\bsTwo{}=0) = \bbP{}(y=1,\bsTwo{}=0) = \gamma - \alpha p 
    \end{align*}
    The occlusion for class $y^*$, where $\bsTwo{}$ serves as the primary evidence, eliminates the case $(y^*=1, \bsTwo{}=1)$, resulting in:
    \begin{align*}
        \bbP{}(y^*=1, \bsTwo{}=1) = (1 - \alpha)p \overset{\textrm{manipulate}}{\implies} \tilde{\bbP{}}(y^*=1, \bsTwo{}=1) = 0
    \end{align*}
    which affects the distribution and leads to the following updated probabilities: 
    \begin{align*}
        & \tilde{\bbP{}}(\bsTwo{}=1) = \alpha p &
        & \tilde{\bbP{}}(\bsTwo{}=0) = 1 - \alpha p \\
        & \tilde{\bbP{}}(y=0,\bsTwo{}=1) = 0 &
        & \tilde{\bbP{}}(y=0,\bsTwo{}=0) = 1 - \gamma \\
    \end{align*}
    The updated conditional probabilities are now given by:
    \begin{align*}
        & \tilde{\bbP{}} (y=1 | \bsTwo{}=1) = 1 &
        & \tilde{\bbP{}} (y=1 | \bsTwo{}=0) = \frac{\gamma - \alpha p}{1 - \alpha p} \\
        & \tilde{\bbP{}} (y=0 | \bsTwo{}=1) = 0 &
        & \tilde{\bbP{}} (y=0 | \bsTwo{}=0) = \frac{1 - \gamma}{1 - \alpha p} 
    \end{align*}
    So far, both $H(y|\bsTwo{})$ and $\tilde{H}(y|\bsTwo{})$ can be expressed in the form of $\gamma$, $p$, and $\alpha p$, namely:
    \begin{align*}
        H(y|\bsTwo{}) = - \Bigl[ 
            \alpha p &\cdot \log_C \alpha  
            + (\gamma - \alpha p) \cdot \log_C \frac{\gamma - \alpha p}{1-p} \\
            &+ (1 - \alpha) p \cdot \log_C (1-\alpha)
            + (1 - \gamma - p + \alpha p ) \cdot \log_C \frac{1 - \gamma - p + \alpha p }{1-p}
        \Bigr] \numberthis\label{eq:orgCondEntropy}
    \end{align*}
    and:
    \begin{align*}
        \tilde{H}(y|\bsTwo{}) &= - \Bigl[ 
            \alpha p \cdot \log_C 1
            + (\gamma - \alpha p) \cdot \log_C \frac{\gamma - \alpha p}{1-\alpha p} 
            + 0 \cdot \log_C 0
            + (1 - \gamma ) \cdot \log_C \frac{1 - \gamma }{1-\alpha p}
        \Bigr] \\
        &= - \Bigl[ (\gamma - \alpha p) \cdot \log_C \frac{\gamma - \alpha p}{1-\alpha p} + (1 - \gamma ) \cdot \log_C \frac{1 - \gamma }{1-\alpha p} \Bigr] 
        \numberthis\label{eq:altCondEntropy}
    \end{align*}
    Combining \eqref{eq:orgCondEntropy} and \eqref{eq:altCondEntropy} give rise to:
    \begin{alignat*}{2}
        H(y|\bsTwo{}) - \tilde{H}(y|\bsTwo{}) & = - \Bigl[ 
            \alpha p &&\cdot \log_C \alpha  
            + (\gamma - \alpha p) \cdot \log_C \frac{\gamma - \alpha p}{1-p} \\
            &&&+ (1 - \alpha) p \cdot \log_C (1-\alpha)
            + (1 - \gamma - p + \alpha p ) \cdot \log_C \frac{1 - \gamma - p + \alpha p }{1-p}
        \Bigr] \\
        &&&+
        \Bigl[ (\gamma - \alpha p) \cdot \log_C \frac{\gamma - \alpha p}{1-\alpha p} + (1 - \gamma ) \cdot \log_C \frac{1 - \gamma }{1-\alpha p} \Bigr] \\
        & = - \Bigl[ 
            \rlap{$\underbrace{\phantom{\alpha p \cdot \log_C \alpha  + (1 - \alpha) p \cdot \log_C (1-\alpha)}}_{\Circled{1}}$}
            \alpha p &&\cdot \log_C \alpha  + (1 - \alpha) p \cdot \log_C (1-\alpha) \\
            &&& \underbrace{+ (\gamma - \alpha p) \cdot \log_C \frac{\gamma - \alpha p}{1-p}  - (\gamma - \alpha p) \cdot \log_C \frac{\gamma - \alpha p}{1-\alpha p}}_{\Circled{2}} \\
            &&& \underbrace{+ (1 - \gamma - p + \alpha p ) \cdot \log_C \frac{1 - \gamma - p + \alpha p }{1-p} - (1 - \gamma ) \cdot \log_C \frac{1 - \gamma }{1-\alpha p}}_{\Circled{3}}
        \Bigr] \tag{Reorganization}
        % & = - \Bigl[ 
        %     \alpha p &&\cdot \log_C \alpha  + (1 - \alpha) p \cdot \log_C (1-\alpha) \\
        %     &&& + (\gamma - \alpha p) \cdot \log_C \frac{1-\alpha p}{1-p}
        % \Bigr] \\
    \end{alignat*}
    Term $\Circled{2}$ simplifies to:
    \begin{align*}
        \Circled{2} &= (\gamma - \alpha p) \cdot \log_C \frac{1-\alpha p}{1-p} \\
        &= (\gamma - \alpha p) \cdot \underbrace{\log_C (1-\alpha p)}_{\textrm{Pair 1}} - \underbrace{(\gamma - \alpha p) \cdot \log_C (1-p)}_{\textrm{Pair 2}}
    \end{align*}
    Term $\Circled{3}$ can be reorganized as:
    \begin{alignat*}{2}
        \Circled{3} &= (1 &&- \gamma - p + \alpha p ) \cdot \log_C (1 - \gamma - p + \alpha p) 
        - \Bigl[(1 - p) - \rlap{$\underbrace{\phantom{(\gamma - \alpha p )\Bigr] \cdot \log_C (1-p) }}_{\textrm{Pair 2}}$}
        (\gamma - \alpha p )\Bigr] \cdot \log_C (1-p) \\
        & &&- (1 - \gamma ) \cdot \log_C (1 - \gamma ) + (1 - \gamma ) \cdot \underbrace{\log_C (1-\alpha p)}_{\textrm{Pair 1}}
        % \Circled{3} &= (1 - \gamma - p + \alpha p ) \cdot \log_C \frac{1 - \gamma - p + \alpha p }{1-p} - (1 - \gamma ) \cdot \log_C \frac{1 - \gamma }{1-\alpha p}
    \end{alignat*}
    Term $\Circled{2}$ and $\Circled{3}$ share common components.
    Collecting all terms and aligning the shared parts gives:
    \begin{alignat*}{2}
        \Circled{2} + \Circled{3} &= (1 &&- \gamma - p + \alpha p ) \cdot \log_C (1 - \gamma - p + \alpha p) - (1-\gamma) \cdot \log_C (1-\gamma)\\
        & && - \Bigl[(1 - p) - \rlap{$\underbrace{\phantom{(\gamma - \alpha p )\Big[ \cdot \log_C (1-p) }}_{\textrm{Pair 2 in }\Circled{3}}$}
        (\gamma - \alpha p )\Bigr] \cdot \log_C (1-p) 
        - \rlap{$\underbrace{\phantom{(\gamma - \alpha p) \Big] \log_C (1-p)}}_{\textrm{Pair 2 in }\Circled{2}}$} (\gamma - \alpha p) \cdot \log_C (1-p)\\
        & && + (1-\gamma) \cdot \underbrace{\log_C (1-\alpha p)}_{\textrm{Pair 1 in }\Circled{3}} + (\gamma - \alpha p) \cdot \underbrace{\log_C (1-\alpha p)}_{\textrm{Pair 1 in }\Circled{2}}    \\
        &= \Bigl[(1 &&- \gamma - p + \alpha p ) \cdot \log_C (1 - \gamma - p + \alpha p) - 1-\gamma) \cdot \log_C (1-\gamma)\Bigr]\\
        & && - \Bigl[(1 - p) \cdot \log_C (1-p) + (1 - \alpha p) \cdot \log_C (1 - \alpha p)\Bigr]\\
        &= \Bigl[(1 &&- \gamma - p + \alpha p ) \cdot \log_C (1 - \gamma - p + \alpha p) - (1-\gamma) \cdot \log_C (1-\gamma)\Bigr]\\
        & && - \Bigl[(1 - p ) \cdot \log_C (1-p ) + (1 - \alpha p \underbrace{+p-p}_{\textrm{Zero term}}) \cdot \log_C (1 - \alpha p \underbrace{+ p - p}_{\textrm{Zero term}}) \Bigr] \numberthis\label{eq:deriveDiff}
    \end{alignat*}
    Let $\delta=p-\alpha p$ be a shorthand for notational convenience. We rewrite \eqref{eq:deriveDiff} as:
    \begin{alignat*}{2}
        \Circled{2} + \Circled{3} &= \Bigl[(1 &&- \gamma - \delta ) \cdot \log_C (1 - \gamma - \delta) - (1-\gamma) \cdot \log_C (1-\gamma)\Bigr]\\
        & && + \Bigl[(1 - p + \delta) \cdot \log_C (1 - p + \delta) - (1 - p ) \cdot \log_C (1-p) \Bigr] \\
        &\approx - \delta &&\cdot \log_C (1-\gamma) + \delta \cdot \log_C (1 - p) \tag{First-order Taylor Expansion} \\
        &= ~\delta ~\cdot && \log_C \frac{1-p}{1-\gamma} 
    \end{alignat*}
    Given $p$ corresponding to the probability of shared features that span over more than one class, we have $p>\gamma$, which derives:
    \begin{align*}
        \Circled{2} + \Circled{3} = \delta \cdot \log_C \frac{1-p}{1-\gamma} < 0
    \end{align*}
    Finally, by collecting all terms --- $\Circled{1}$, $\Circled{2}$, and $\Circled{3}$ --- we conclude the proof:
    \begin{align*}
        H(y|\bsTwo{}) - \tilde{H}(y|\bsTwo{}) &= - \Bigl[ \Circled{1} + \Circled{2} + \Circled{3} \Bigr] \\
        &= - \Bigl[ \underbrace{\alpha p \cdot \log_C \alpha}_{<0}  + \underbrace{(1 - \alpha) p \cdot \log_C (1-\alpha) }_{<0} + \underbrace{\delta \cdot \log_C \frac{1-p}{1-\gamma}}_{<0} \Bigr]  \\
        &> 0 \\
        \implies \tilde{I}(\bsTwo{};y) &> I(\bsTwo{};y)
    \end{align*}
\end{proof}

The above proof focuses on the case where the shared features are associated with only two classes. 
However, the same derivation can be readily generalized to more complicated settings involving features shared across multiple classes.
Specifically, the relevant classes can be split into two groups based on the attribution sign of the shared features.
The derivation can then be applied to the subset of classes for which the target features serve as secondary evidence, demonstrating the increase in utility of these features for that specific group. 

\subsection{Proof of Negative Attribution by Weak Positive Contributor}\label{apx:proof_wpc}
\begin{proof}[Proof of Theorem~\ref{thm:wpc}]
    For a classification task with $C$ classes, let $y$ be the target class, to which the \textit{i}-th feature is a weak positive contributor. 
    The relationship between raw model outcomes (by the classification head) and model confidence (by the softmax layer) is:
    \begin{equation*}
        f_y(\boldsymbol{N})=\sigma\Bigl(o_y(\boldsymbol{N})\Bigr)
        =\frac{\exp\Bigl(
            o_y(\boldsymbol{N})
        \Bigr)}
        {\sum_{k=1}^{C}\exp\Bigl(
            o_k(\boldsymbol{N})
        \Bigr)}
    \end{equation*}
    Similarly, the relationship conditioned by the absence of the \textit{i}-th feature is:
    \begin{equation*}
        f_y(\boldsymbol{N}\backslash \{i\})=\sigma\Bigl(o_y(\boldsymbol{N}\backslash \{i\})\Bigr)=
        \frac{\exp\Bigl( 
            o_y(\boldsymbol{N}\backslash \{i\}) 
        \Bigr)}
        {\sum_{k=1}^{C}\exp\Bigl( 
            o_k(\boldsymbol{N}\backslash \{i\}) 
        \Bigr)}
    \end{equation*}
    Please note that defining feature absence $\boldsymbol{N}\backslash \{i\}$ is a complicated topic in the context of machine learning.
    As machine learning models do not process incomplete inputs, current solutions usually consider some constant default values or represent the absence state by sampling substitutes from some carefully designed distributions.
    To better focus on the discussion of the sign issue, here we assume that the perfect absence state is available, which facilitates the derivation of the underlying feature contributions:
    \begin{equation}\label{eq:add_contribution}
        o_y(\boldsymbol{N})=o_y(\boldsymbol{N}\backslash \{i\}) + \xi_i^{o_y}
    \end{equation}
    For a feature to receive a negative attribution score w.r.t. to model confidence $f_y(\cdot)$, it requires:
    \begin{align*}
        f_y(\boldsymbol{N})&<f_y(\boldsymbol{N}\backslash \{i\})\\
        \implies~\frac{\exp\Bigl(
            o_y(\boldsymbol{N})
        \Bigr)}
        {\sum_{k=1}^{C}\exp\Bigl(
            o_k(\boldsymbol{N})
        \Bigr)} &< \frac{\exp\Bigl(
            o_y(\boldsymbol{N}\backslash \{i\})
        \Bigr)}
        {\sum_{k=1}^{C}\exp\Bigl(
            o_k(\boldsymbol{N}\backslash \{i\})
        \Bigr)}
    \end{align*}
    Applying Eq.~\ref{eq:add_contribution} and reorganizing the elements derives:
    \begin{equation*}
        \exp( \xi_i^{o_y} ) 
        < \frac{\sum_{k=1}^{C}\exp\Bigl(
            o_k(\boldsymbol{N}\backslash \{i\}) + \xi_i^{o_c} 
        \Bigr)}{\sum_{k=1}^{C}\exp\Bigl(o_k(\boldsymbol{N}\backslash \{i\})\Bigr)}
        = \frac{\exp\Bigl(
            o_y(\boldsymbol{N}\backslash \{i\}) + \xi_i^{o_y}
        \Bigr) + 
        \sum_{y^*\neq y}\exp\Bigl(
            o_{y^*}(\boldsymbol{N}\backslash \{i\}) + \xi_i^{o_{y^*}}
        \Bigr)}
        {\exp\Bigl(
            o_y(\boldsymbol{N}\backslash \{i\})
        \Bigr) + 
        \sum_{y^* \neq y}\exp\Bigl(
            o_{y^*}(\boldsymbol{N}\backslash \{i\})
        \Bigr)}
    \end{equation*}
    We replace the class label subscript $k$ with $y^*$ on the right-hand side to emphasize a semantic distinction. 
    The symbol $k$ denotes an enumeration over \textit{all} classes, whereas $y^*$ specifically refers to \textit{classes} other than the target class $y$, namely $y^*\neq y$.
    Multiplying both sides by the denominator further simplifies the above inequality to:
    \begin{align*}
        \exp( \xi_i^{o_y} ) \cdot
        \Biggl[
            \exp\Bigl(
                o_y(\boldsymbol{N}\backslash \{i\})
            \Bigr) + 
            \sum_{y^* \neq y}\exp\Bigl(
                o_{y^*}(\boldsymbol{N}\backslash \{i\})
            \Bigr)
        \Biggr]
        &< \exp\Bigl(
            o_{y}(\boldsymbol{N}\backslash \{i\}) + \xi_i^{o_y}
        \Bigr) + \sum_{y^*\neq y}\exp\Bigl(
            o_{y^*}(\boldsymbol{N}\backslash \{i\}) + \xi_i^{o_{y^*}} 
        \Bigr) \\
        {\implies}
        \exp( \xi_i^{o_y} ) \cdot
        \sum_{y^* \neq y}\exp\Bigl( o_{y^*}(\boldsymbol{N}\backslash \{i\}) \Bigr)
        &< \sum_{y^*\neq y}\exp\Bigl(
            o_{y^*}(\boldsymbol{N}\backslash \{i\}) + \xi_i^{o_{y^*}} 
        \Bigr) \\
        \implies 
        \exp( \xi_i^{o_y}) 
        &< \frac{\sum_{y^*\neq y}\exp\Bigl(
            o_{y^*}(\boldsymbol{N}\backslash \{i\})
        \Bigr) 
        \cdot \exp( \xi_i^{o_{y^*}} )} 
        {\sum_{y^*\neq y}\exp\Bigl(
            o_{y^*}(\boldsymbol{N}\backslash \{i\})
        \Bigr)}
    \end{align*}
    The second line in the above derivation follows from canceling out the identical first term on both sides.
    Now we rewrite the right-hand side by dividing both the numerator and the denominator by $\sum_{k=1}^C \exp\Bigl(o_k(\boldsymbol{N}\backslash \{i\})\Bigr)$, yielding:
    \begin{equation*}
        \exp( \xi_i^{o_y} ) 
        < \sum_{y^*\neq y} \frac{ f_{y^*}(\boldsymbol{N}\backslash \{i\})} 
        {1 - f_y(\boldsymbol{N}\backslash \{i\})}
        \cdot \exp( \xi_i^{o_{y^*}} )
    \end{equation*}
    This concludes the proof by interpreting $\frac{ f_{y^*}(\boldsymbol{N}\backslash \{i\})}{1 - f_y(\boldsymbol{N}\backslash \{i\})}$ as the probability of class $y^*$ conditioned on $y^*\neq y$, namely:
    \begin{equation*}
        \exp( \xi_i^{o_y} ) 
        < \mathbb{E}_{y^*\neq y}\Bigl[\exp( \xi_i^{o_{y^*}} )\Bigr]
    \end{equation*}
    
\end{proof}

\subsection{Extended Discussion on Related Work}\label{apx:extended_rw}
\citet{hooker2019benchmark} first mentioned \kear{} (referred to as \textsc{kar} for \underline{k}eep \underline{a}nd \underline{r}etrain) in the appendix of their paper as an alternative to \roar{} that differs in occlusion priority.
The results concluded that \kear{} is  ``a poor discriminator between estimators''.
However, we argue that this conclusion is misleading, as it results from mixing evaluation outcomes across different attribution scales. 
Specifically, \citet{hooker2019benchmark} included two variants of IG as competitors --- SQ-IG and Var-IG --- both of which convert signed attribution scores to magnitudes. 
Consequently, the \roar{} assessments for these unsigned attributions are equivalent to \roarAbs{}, which practically employs a relevant-first occlusion rather than the expected highest-first strategy.

As detailed in Section~\ref{sec:kaft}, relevant-first occlusion mitigates the sign issue and is therefore less susceptible to residual information, which distorts retraining-based evaluations.
As a result, comparing \roar{} results on \textit{signed} attributions (where residual information remains) with the results on \textit{unsigned} attributions (where residuals are mitigated) unintentionally inflates its ability to distinguish explainer performances.
This inflation perceptually undermines the significance of \kear{}'s results.
In contrast, our results in Figure~\ref{fig:smallScal90p} and Table~\ref{apx:tab:small_stats_kaft}, derived from consistent experimental settings, demonstrate that \kear{} is an effective evaluation scheme: it not only discriminates well among tested explainers but also mitigates evaluation distortions, aligning with theoretical expectations.

% TODO: discuss results on |KAFT-C|

\subsection{Detailed Experimental Settings}\label{apx:exp_details}
\subsubsection{Abbreviations and Naming Rules of Evaluation Scheme}\label{apx:abbrvs}
Table~\ref{tab:abbrvs} summarizes the abbreviations used for different evaluation schemes under the retraining framework.
An evaluation scheme abbreviation used in this work is composed of three components:
\begin{enumerate}
    \item \textbf{Prefix}\footnote{Please note that, for the two standard retraining settings, the prefix is \textsc{ro}- and \textsc{ke}- in \roar{} and \kear{} to be consistent with the naming policy in \citep{hooker2019benchmark}.}: An abbreviation is started by either \textsc{r}- or \textsc{k}-, depending on the ordering of feature removal, where \textsc{r}- corresponds to highest-first and \textsc{k}- corresponds to lowest-first removal. 
    \item \textbf{Postfix}: The ending of an abbreviation states the exact retraining strategy. 
    The choices include retraining (-\textsc{r}), fine-tuning (-\textsc{ft}), and fine-tuning on the classification head (-\textsc{ft-c}).
    \item \textbf{The Magnitude-only Flag}: The absolute value operator $\vert\cdot\vert$ indicates a manipulation process driven by attribution magnitudes. 
    The combination of magnitude-based occlusion with restrictive fine-tuning on classification head offers an alternative to resolve the issue of residual information
    %   Concentrating on attribution magnitudes offers alternatives to resolve the sign issue.
    %   Although magnitude-based occlusion can still be sensitive to residual information arise from feature redundancy, combining 
\end{enumerate}

\begin{table}
    \centering
    \caption{Evaluation strategy abbreviations and corresponding schemes}\label{tab:abbrvs}
    \begin{tblr}{c|c}
        \toprule
        Abbrv. & Scheme \\
        \midrule
        \roar{}/\kear{} & Remove/Keep and Retraining \\
        \raft{}/\kaft{} & Remove/Keep and Fine-tuning  \\
        \raftc{}/\kaftc{} & Remove/Keep and Fine-tuning on Classification Head \\
        \raftcAbs{}/\kaftcAbs{} & F.-t. on Classification Head following attribution magnitude \\
        \bottomrule
    \end{tblr}
\end{table}

\subsubsection{Experimental Environment}
The competitors, including the proposed method, were implemented using Python 3.11.2 with standard packages.
The primary packages were Numpy 1.24.2, PyTorch of version 2.6.0, and Torchvision 0.21.0. 
The CUDA version was 12.5 for GPU support. 
All experiments were conducted on a machine operated by Debian 11 with the following specifications:
\begin{itemize}
    \item Processor: Intel i7-14700K, 20 cores
    \item Memory: 64GB DDR4
    \item GPU: NVIDIA RTX 5000 ADA, 32GB
\end{itemize}

\subsubsection{Feature Attribution Methods}
During the experiments, we focused on gradient-based feature attribution methods for their efficiency and generalizability.
The choices, including the extension in the large-scale settings, are:
\begin{itemize}
    \item Vanilla Gradient (\textbf{VG}) interprets model sensitivity to input features (as reflected by raw gradients) directly as explanations to decisions:
    \[ \xi_i = \frac{\partial f_y(\boldsymbol{x})}{\partial x_i} \]
    \item \textsc{SmoothGrad} (\textbf{SG}) improves over vanilla gradient by smoothing gradients through an average of \textit{k} Gaussian-noised copies of the input:
    \[ \xi_i = \frac{1}{k} \sum_{\boldsymbol{\epsilon}\sim\mathcal{N}} \frac{\partial f_y(\boldsymbol{x})}{\partial x_i}\bigg\rvert_{\boldsymbol{x}=\boldsymbol{x}+\boldsymbol{\epsilon}} \]
    \item Integrated Gradients (\textbf{IG}) incorporates a baseline $\boldsymbol{\mathring{x}}$ that models feature absence during the explanation procedure and integrates gradients over an interpolated path between the absence and presence states:
    \[ \xi_i = \frac{(\boldsymbol{x} - \boldsymbol{\mathring{x}})}{k} \sum_{i=0}^{k} \frac{\partial f_y(\boldsymbol{x})}{\partial x_i}\bigg\rvert_{\boldsymbol{\mathring{x}} + \frac{i}{k}(\boldsymbol{x} - \boldsymbol{\mathring{x}})} \]
    \item Gradient$\odot$Input (\textbf{G$\odot$I}) improves the sharpness of attribution maps by multiplying raw gradients with the corresponding input. G$\odot$I can be viewed as a special case of IG, where only one observation is collected:
    \[ \xi_i = x_i \cdot \frac{\partial f_y(\boldsymbol{x})}{\partial x_i} \]
    \item Expected Gradients (\textbf{EG}) is a variant of IG, modeling the feature absence with a baseline distribution $\mathcal{B}$ rather than constant values:
    \[ \xi_i = \mathbb{E}_{\boldsymbol{\mathring{x}}\sim\mathcal{B}}\bigg[\frac{(\boldsymbol{x} - \boldsymbol{\mathring{x}})}{k} \sum_{i=0}^{k} \frac{\partial f_y(\boldsymbol{x})}{\partial x_i}\bigg\rvert_{\boldsymbol{\mathring{x}} + \frac{i}{k}(\boldsymbol{x} - \boldsymbol{\mathring{x}})} \bigg] \]
    \item Smoothed Integrated Gradients (\textbf{SIG}) is a common ensemble that combines the theoretical strength of IG with the denoising effect provided by SG:
    \[ \xi_i = \frac{(\boldsymbol{x} - \boldsymbol{\mathring{x}})}{k} \sum_{i=0, \boldsymbol{\epsilon}\sim\mathcal{N}}^{k} \frac{\partial f_y(\boldsymbol{x})}{\partial x_i}\bigg\rvert_{\boldsymbol{\mathring{x}} + \frac{i}{k}(\boldsymbol{x} - \boldsymbol{\mathring{x}}) + \boldsymbol{\epsilon}} \]
\end{itemize}

\subsubsection{Classifiers and Training Settings}
A CNN composed of two convolutional layers concatenated by two dense layers is trained for MNIST;
the variant WideResNet 40-4 is trained on CIFAR10 from scratch;
the variant EfficientNet-B0\footnote{https://docs.pytorch.org/vision/main/models/efficientnet.html} is trained on STL10 from scratch.
Table~\ref{apx:tab:train} summarizes the training configurations for the small-scale settings.
Since the scope of this paper is on explanation evaluation, the training hyperparameters are configured manually to achieve a moderate performance, and there could be space for further improvements of model performance.
Regarding the large-scale tests, the pretrained versions of ResNet50\footnote{https://docs.pytorch.org/vision/main/models/resnet.html} and SwinT\footnote{https://docs.pytorch.org/vision/main/models/swin\_transformer.html} are used for the experiments on ImageNet1k without any further training.
Table~\ref{apx:tab:acc} lists the performances of the tested classifiers on the corresponding dataset.

\begin{table}
  \centering
  \caption{Training settings for small-scale tests}\label{apx:tab:train}
  \begin{tblr}{c|ccccc}
    \toprule
    Model & Epochs & Optim. & LR & Sched. & Augment. \\
    \midrule
    CNN & 30 & SGD & $0.01$ & - & - \\
    WideResNet 40-4 & 80 & SGD & $0.1$ & MultiStepLR & RC$^{a}$, HF$^{b}$\\
    EfficientNet-B0 & 300 & Adam & $2\times 10^{-3}$ & CosineLR & RC, HF \\[0.5mm]
    \bottomrule
    \SetCell[c=3]{l}{$^{a}$Random Crop} & & & \SetCell[c=3]{l}{$^{b}$Horizontal Flip} &&
  \end{tblr}
\end{table}

\begin{table}
  \centering
  \caption{Model performance on corresponding dataset}\label{apx:tab:acc}
  \begin{tblr}{c|ccc|cc}
    \toprule
    Dataset & MNIST & CIFAR10 & STL10 & \SetCell[c=2]{c}{ImageNet1k} \\
    Model & CNN & WideResNet & EfficientNet & ResNet50 & SwinT \\[0.5mm]
    \midrule
    Acc. ($\%$) & 99.04 & 94.43 & 80.91 & 76.13$^{c}$ & 81.47$^{c}$ \\[0.5mm]
    \bottomrule
    \SetCell[c=6]{l}{$^{c}$Top-1 accuracy reported by the model provider.} 
  \end{tblr}
\end{table}

During retraining, the exact training setting is used for each test case.
When evaluated by \kaft{}, the number of epochs is reduced to 30 for fine-tuning with one warmup epoch; a cosine scheduler is applied to update the learning rate, and the rest of the settings remain unchanged.
For fine-tuning on the classification head, the number of epochs is further reduced to 10. 

\subsection{Additional Experimental Results}\label{apx:additional_results}

\subsubsection{Detailed Experimental Results}

\begin{table}[t]
    \centering
    \caption{Evaluation results by \roar{}\flatdownarrow{} reported on small-scale settings}\label{apx:tab:small_stats_roar}
    \begin{tblr}{c|c|ccc}
        \toprule
        \roar{}~\flatdownarrow & Random & VG & SG & IG \\
        \midrule
        MNIST-CNN & 78.62 ± 0.29 & \textbf{88.93 ± 0.26} & 89.30 ± 0.64 & 97.02 ± 0.13 \\
        CIFAR10-WideResNet & 63.34 ± 1.79 & \textbf{74.05 ± 0.72} & 84.16 ± 0.72 & 79.54 ± 1.58 \\
        STL10-EfficientNet & 60.19 ± 0.83 & 62.16 ± 0.93 & 65.85 ± 0.25 & \textbf{60.66 ± 0.24} \\
        \bottomrule
        \SetCell[c=5]{l}{\flatdownarrow: Lower is better. }
    \end{tblr}
\end{table}

\begin{table}[t]
    \centering
    \caption{Evaluation results by \kaft{}\flatuparrow{} variants reported on small-scale settings}\label{apx:tab:small_stats_kaft}
    \begin{tblr}{c|c|ccc}
        \toprule
        \kear{}~\flatuparrow & Random & VG & SG & IG \\
        \midrule
        MNIST-CNN & 78.62 ± 0.29 & 94.38 ± 0.14 & 97.34 ± 0.17 & \textbf{98.99 ± 0.04} \\
        CIFAR10-WideResNet & 63.34 ± 1.79 & 72.58 ± 0.69 & 83.45 ± 0.59 & \textbf{83.86 ± 1.10} \\
        STL10-EfficientNet & 60.19 ± 0.83 & 60.42 ± 0.63 & 62.71 ± 0.26 & \textbf{71.99 ± 0.21} \\
        \bottomrule
        \SetCell[c=5]{c} \\[-50pt]
        \toprule
        \kaft{}~\flatuparrow & Random & VG & SG & IG \\
        \midrule
        MNIST-CNN & 55.95 ± 2.23 & 87.55 ± 0.64 & 94.77 ± 0.09 & \textbf{98.95 ± 0.06} \\
        CIFAR10-WideResNet & 42.17 ± 1.33 & 46.35 ± 2.32 & 62.52 ± 1.49 & \textbf{69.99 ± 2.60} \\
        STL10-EfficientNet & 33.86 ± 1.19 & 36.26 ± 1.71 & 37.01 ± 2.32 & \textbf{42.57 ± 5.08} \\
        \bottomrule
        \SetCell[c=5]{c} \\[-50pt]
        \toprule
        \kaftc{}~\flatuparrow & Random & VG & SG & IG \\
        \midrule
        MNIST-CNN & 32.87 ± 0.24 & 59.15 ± 0.21 & 74.22 ± 0.19 & \textbf{96.47 ± 0.03} \\
        CIFAR10-WideResNet & 19.12 ± 0.38 & 19.81 ± 0.56 & 23.89 ± 0.17 & \textbf{31.57 ± 0.05} \\
        STL10-EfficientNet & 15.30 ± 0.55 & 19.23 ± 0.88 & 19.59 ± 0.70 & \textbf{36.26 ± 1.13} \\
        \bottomrule
        \SetCell[c=5]{l}{\flatuparrow: Higher is better. }
    \end{tblr}
\end{table}

\begin{table}[tbp]
    \centering
    \caption{Evaluation results by \raftcAbs{}\flatdownarrow{} reported on small-scale settings}\label{apx:tab:small_stats_raft}
    \begin{tblr}{c|c|ccc}
        \toprule
        \raftcAbs{}~\flatdownarrow & Random & VG & SG & IG \\
        \midrule
        MNIST-CNN & 32.87 ± 0.24 & 22.07 ± 0.37 & 23.05 ± 0.51 & \textbf{11.35 ± 0.00} \\
        CIFAR10-WideResNet & 19.12 ± 0.38 & 20.04 ± 0.35 & 20.30 ± 0.20 & \textbf{19.65 ± 0.34} \\
        STL10-EfficientNet & 15.30 ± 0.55 & 13.20 ± 0.76 & 14.50 ± 0.50 & \textbf{13.16 ± 0.47} \\
        \bottomrule
    \end{tblr}
\end{table}

Tables~\ref{apx:tab:small_stats_roar},~\ref{apx:tab:small_stats_kaft}, and~\ref{apx:tab:small_stats_raft} present the detailed evaluation results obtained following different evaluation schemes under the three \textit{small-scale settings}.
Each reported value represents the mean accuracy over 5 retraining repetitions for the reliability of the results, followed by the standard deviation. 
Notable, the performance of the random baseline is identical for \roar{} (Table~\ref{apx:tab:small_stats_roar}) and \kear{} (the first block in Table~\ref{apx:tab:small_stats_kaft}), as the two schemes differ only in occlusion priority --- a factor to which random removal is insensitive.
The same holds for the random baseline of \kaftc{} (the third block in Table~\ref{apx:tab:small_stats_kaft}) and \raftcAbs{} (Table~\ref{apx:tab:small_stats_raft}).

Tables~\ref{apx:tab_kaft_resnet},~\ref{apx:tab_raft_resnet},~\ref{apx:tab_kaft_swint}, and~\ref{apx:tab_raft_swint} report the evaluation statistics of \kaftc{} and \raftcAbs{} on ImageNet-ResNet50 and ImageNet-SwinT, complementing the visual results shown in Figure~\ref{fig:ratioLarge}.

\begin{table}
    \centering
    \caption{Evaluation results by \kaftc{}\flatuparrow{} on \textbf{ImageNet-ResNet50} under various removal ratios (ranging from $10\%$ to $90\%$)}\label{apx:tab_kaft_resnet}
    \begin{tblr}{width=0.95\textwidth, colspec={c|X[1,c]|X[1,c]X[1,c]X[1,c]X[1,c]}}
        % \SetCell[c=6]{l} \kaft{}~\flatuparrow, ImageNet-ResNet50 &&&&&\\
        \toprule
        Edit Ratio & Original & $10\%$ & $20\%$ & $30\%$ & $40\%$ \\
        \midrule
        Random & \SetCell[r=7]{c} 76.13 & 60.92 ± 0.08 & 54.46 ± 0.19 & 48.96 ± 0.19 & 43.27 ± 0.17 \\
        VG & & 53.80 ± 0.11 & 46.36 ± 0.20 & 40.69 ± 0.18 & 36.37 ± 0.19 \\
        SG & & 50.53 ± 0.20 & 42.92 ± 0.16 & 37.62 ± 0.21 & 33.99 ± 0.07 \\
        G$\odot$I && 60.75 ± 0.13 & 54.67 ± 0.08 & 49.65 ± 0.15 & 46.20 ± 0.09 \\
        IG && 63.22 ± 0.10 & 59.14 ± 0.07 & 55.15 ± 0.22 & 51.89 ± 0.33 \\
        EG && 63.02 ± 0.16 & 58.13 ± 0.07 & 54.13 ± 0.19 & 51.13 ± 0.02 \\
        SIG && 65.92 ± 0.05 & 62.67 ± 0.13 & 59.78 ± 0.13 & 57.55 ± 0.26 \\
        \bottomrule
        \SetCell[c=5]{c} \\[-50pt]
    \end{tblr}

    \begin{tblr}{width=0.95\textwidth, colspec={c|X[1,c]X[1,c]X[1,c]X[1,c]X[1,c]}}
        \toprule
        Edit Ratio & $50\%$ & $60\%$ & $70\%$ & $80\%$ & $90\%$ \\
        \midrule
        Random & 38.11 ± 0.14 & 32.69 ± 0.20 & 26.82 ± 0.17 & 20.24 ± 0.05 & 12.98 ± 0.08 \\
        VG & 32.43 ± 0.24 & 28.07 ± 0.16 & 24.17 ± 0.29 & 19.59 ± 0.05 & 13.06 ± 0.17 \\
        SG & 31.85 ± 0.18 & 29.76 ± 0.16 & 27.46 ± 0.01 & 24.15 ± 0.03 & 18.08 ± 0.25 \\
        G$\odot$I & 43.40 ± 0.09 & 40.15 ± 0.10 & 33.96 ± 0.09 & 26.02 ± 0.03 & 15.60 ± 0.15 \\
        IG & 50.13 ± 0.13 & 47.65 ± 0.09 & 43.70 ± 0.19 & 37.09 ± 0.10 & 24.74 ± 0.18 \\
        EG & 49.02 ± 0.24 & 45.93 ± 0.09 & 41.76 ± 0.20 & 35.02 ± 0.14 & 22.10 ± 0.04 \\
        SIG & 56.09 ± 0.04 & 55.09 ± 0.13 & 52.76 ± 0.15 & 47.86 ± 0.07 & 35.61 ± 0.18 \\
        \bottomrule
    \end{tblr}
\end{table}

\begin{table}
    \centering
    \caption{Evaluation results by \raftcAbs{}\flatdownarrow{} on \textbf{ImageNet-ResNet50} under various removal ratios (ranging from $10\%$ to $90\%$)}\label{apx:tab_raft_resnet}
    \begin{tblr}{width=0.95\textwidth, colspec={c|X[1,c]|X[1,c]X[1,c]X[1,c]X[1,c]}}
        \toprule
        Edit Ratio & Original & $10\%$ & $20\%$ & $30\%$ & $40\%$ \\
        \midrule
        Random & \SetCell[r=7]{c} 76.13 & 60.92 ± 0.08 & 54.46 ± 0.19 & 48.96 ± 0.19 & 43.27 ± 0.17 \\
        VG && 55.24 ± 0.08 & 45.94 ± 0.11 & 38.93 ± 0.21 & 32.74 ± 0.25 \\
        SG && 48.49 ± 0.05 & 36.57 ± 0.01 & 28.70 ± 0.14 & 22.64 ± 0.15 \\
        G$\odot$I && 53.51 ± 0.04 & 43.88 ± 0.04 & 35.86 ± 0.10 & 29.21 ± 0.06 \\
        IG && 50.65 ± 0.16 & 40.04 ± 0.16 & 32.34 ± 0.14 & 25.72 ± 0.13 \\
        EG && 51.30 ± 0.23 & 40.07 ± 0.03 & 31.53 ± 0.19 & 25.17 ± 0.16 \\
        SIG && 46.86 ± 0.07 & 35.35 ± 0.03 & 28.03 ± 0.23 & 22.26 ± 0.12 \\
        \bottomrule
        \SetCell[c=5]{c} \\[-50pt]
    \end{tblr}

    \begin{tblr}{width=0.95\textwidth, colspec={c|X[1,c]X[1,c]X[1,c]X[1,c]X[1,c]}}
        \toprule
        Edit Ratio & $50\%$ & $60\%$ & $70\%$ & $80\%$ & $90\%$ \\
        \midrule
        Random & 38.11 ± 0.14 & 32.69 ± 0.20 & 26.82 ± 0.17 & 20.24 ± 0.05 & 12.98 ± 0.08 \\
        VG & 27.70 ± 0.09 & 22.93 ± 0.15 & 18.73 ± 0.10 & 14.48 ± 0.16 & 9.41 ± 0.08 \\
        SG & 18.51 ± 0.18 & 15.07 ± 0.12 & 12.88 ± 0.12 & 10.13 ± 0.19 & 7.60 ± 0.19 \\
        G$\odot$I & 24.17 ± 0.16 & 19.56 ± 0.06 & 15.39 ± 0.09 & 11.23 ± 0.11 & 7.21 ± 0.04 \\
        IG & 20.94 ± 0.13 & 17.27 ± 0.21 & 13.48 ± 0.31 & 10.36 ± 0.02 & 6.92 ± 0.08 \\
        EG & 20.08 ± 0.18 & 16.11 ± 0.14 & 12.82 ± 0.17 & 9.37 ± 0.12 & 6.52 ± 0.04 \\
        SIG & 18.15 ± 0.19 & 14.72 ± 0.06 & 11.79 ± 0.11 & 9.02 ± 0.09 & 6.42 ± 0.14 \\
        \bottomrule
    \end{tblr}
\end{table}

\begin{table}
    \centering
    \caption{Evaluation results by \kaftc{}\flatuparrow{} on \textbf{ImageNet-SwinT} under various removal ratios (ranging from $10\%$ to $90\%$)}\label{apx:tab_kaft_swint}
    \begin{tblr}{width=0.95\textwidth, colspec={c|X[1,c]|X[1,c]X[1,c]X[1,c]X[1,c]}}
        % \SetCell[c=6]{l} \kaft{}~\flatuparrow, ImageNet-ResNet50 &&&&&\\
        \toprule
        Edit Ratio & Original & $10\%$ & $20\%$ & $30\%$ & $40\%$ \\
        \midrule
        Random & \SetCell[r=7]{c} 81.47 & 72.34 ± 0.16 & 64.83 ± 0.10 & 58.06 ± 0.12 & 51.62 ± 0.09 \\
        VG && 67.78 ± 0.08 & 61.20 ± 0.18 & 55.09 ± 0.11 & 49.65 ± 0.14 \\
        SG && 60.08 ± 0.27 & 51.99 ± 0.14 & 47.18 ± 0.19 & 43.77 ± 0.08 \\ 
        G$\odot$I && 72.79 ± 0.06 & 67.44 ± 0.11 & 61.98 ± 0.05 & 55.90 ± 0.12 \\
        IG && 73.99 ± 0.17 & 68.93 ± 0.09 & 63.20 ± 0.08 & 57.39 ± 0.13 \\
        EG && 75.23 ± 0.05 & 72.84 ± 0.06 & 70.38 ± 0.13 & 67.75 ± 0.08 \\
        SIG && 75.80 ± 0.04 & 73.63 ± 0.12 & 71.34 ± 0.03 & 69.33 ± 0.13 \\
        \bottomrule
        \SetCell[c=5]{c} \\[-50pt]
    \end{tblr}

    \begin{tblr}{width=0.95\textwidth, colspec={c|X[1,c]X[1,c]X[1,c]X[1,c]X[1,c]}}
        \toprule
        Edit Ratio & $50\%$ & $60\%$ & $70\%$ & $80\%$ & $90\%$ \\
        \midrule
        Random & 45.39 ± 0.12 & 39.76 ± 0.16 & 31.76 ± 0.03 & 19.85 ± 0.11 & 8.01 ± 0.13 \\
        VG & 43.26 ± 0.08 & 36.07 ± 0.13 & 26.47 ± 0.05 & 16.34 ± 0.11 & 7.35 ± 0.10 \\
        SG & 42.01 ± 0.14 & 40.69 ± 0.14 & 39.01 ± 0.11 & 35.88 ± 0.13 & 27.32 ± 0.09 \\
        G$\odot$I & 50.33 ± 0.04 & 43.35 ± 0.05 & 34.37 ± 0.10 & 22.15 ± 0.15 & 9.17 ± 0.21\\
        IG & 52.03 ± 0.09 & 45.99 ± 0.20 & 36.76 ± 0.13 & 23.97 ± 0.02 & 9.92 ± 0.10 \\
        EG & 65.30 ± 0.13 & 61.84 ± 0.20 & 54.54 ± 0.26 & 40.71 ± 0.12 & 19.38 ± 0.02 \\
        SIG & 68.54 ± 0.05 & 67.55 ± 0.03 & 65.43 ± 0.07 & 60.12 ± 0.05 & 42.92 ± 0.08 \\
        \bottomrule
    \end{tblr}
\end{table}

\begin{table}
    \centering
    \caption{Evaluation results by \raftcAbs{}~\flatdownarrow{} on \textbf{ImageNet-SwinT} under various removal ratios (ranging from $10\%$ to $90\%$)}\label{apx:tab_raft_swint}
    \begin{tblr}{width=0.95\textwidth, colspec={c|X[1,c]|X[1,c]X[1,c]X[1,c]X[1,c]}}
        \toprule
        Edit Ratio & Original & $10\%$ & $20\%$ & $30\%$ & $40\%$ \\
        \midrule
        Random & \SetCell[r=7]{c} 81.47 & 72.34 ± 0.16 & 64.83 ± 0.10 & 58.06 ± 0.12 & 51.62 ± 0.09 \\
        VG && 70.99 ± 0.07 & 65.15 ± 0.15 & 60.29 ± 0.07 & 54.43 ± 0.20 \\
        SG && 58.70 ± 0.19 & 46.59 ± 0.10 & 37.46 ± 0.14 & 29.82 ± 0.10 \\
        G$\odot$I && 71.65 ± 0.11 & 65.88 ± 0.08 & 60.30 ± 0.15 & 53.51 ± 0.02 \\
        IG && 71.89 ± 0.05 & 65.82 ± 0.10 & 60.67 ± 0.16 & 54.91 ± 0.12 \\
        EG && 70.63 ± 0.12 & 64.10 ± 0.03 & 57.56 ± 0.09 & 50.36 ± 0.10 \\
        SIG && 60.49 ± 0.03 & 51.98 ± 0.20 & 45.11 ± 0.17 & 38.52 ± 0.17 \\
        \bottomrule
        \SetCell[c=5]{c} \\[-50pt]
    \end{tblr}

    \begin{tblr}{width=0.95\textwidth, colspec={c|X[1,c]X[1,c]X[1,c]X[1,c]X[1,c]}}
        \toprule
        Edit Ratio & $50\%$ & $60\%$ & $70\%$ & $80\%$ & $90\%$ \\
        \midrule
        Random & 45.39 ± 0.12 & 39.76 ± 0.16 & 31.76 ± 0.03 & 19.85 ± 0.11 & 8.01 ± 0.13 \\
        VG & 49.23 ± 0.19 & 42.33 ± 0.04 & 34.03 ± 0.12 & 23.86 ± 0.16 & 13.11 ± 0.08 \\
        SG & 24.14 ± 0.14 & 19.16 ± 0.07 & 14.74 ± 0.10 & 9.95 ± 0.10 & 3.06 ± 0.06 \\
        G$\odot$I & 46.69 ± 0.17 & 39.29 ± 0.09 & 30.48 ± 0.07 & 19.39 ± 0.14 & 8.94 ± 0.12 \\
        IG & 47.98 ± 0.12 & 39.52 ± 0.13 & 31.08 ± 0.06 & 20.62 ± 0.17 & 6.61 ± 0.11 \\
        EG & 42.88 ± 0.22 & 35.28 ± 0.05 & 26.96 ± 0.06 & 16.85 ± 0.08 & 7.47 ± 0.10 \\
        SIG & 32.23 ± 0.18 & 26.22 ± 0.10 & 19.55 ± 0.19 & 13.11 ± 0.08 & 6.80 ± 0.10 \\
        \bottomrule
    \end{tblr}
\end{table}

\subsubsection{Qualitative Examples of Input Manipulation}\label{apx:examples}
Figures~\ref{apx:fig:mnist},~\ref{apx:fig:cifar}, and~\ref{apx:fig:stl} illustrate examples of perturbed inputs after explanation-guided manipulation.
The manipulation process follows the explanations derived by IG, and a manipulation ratio of $90\%$ is used.
For all three figures, the first column lists the target explicand, with the corresponding class label presented to the right of the input image.
The second column visualizes the attributions determined by IG, with red/blue indicating positive/negative and the intensity implying the attribution amplitude.
The last three columns correspond to the manipulated inputs following the \roar{}, \kaftc{}, and \raftcAbs{}, respectively.

The sign issue is emphasized by the target-overlapping negatives observed across all examples --- pixels marked by blue in feature attribution maps fall in the region representing the target object.
When following the highest-first manipulation strategies with attribution signs, residual information is preserved, which distorts evaluation results under \roar{}.
In the last two rows of Figure~\ref{apx:fig:mnist}, the simple examples visualize how isolating negative features decomposes different information representations.
For the second-to-last row, the negative features decompose a representation of ``1'' from ``7'', and similarly for the last row, where a representation of ``4'' is composed from ``9''.

By contrast, the lowest-first removal adopted by the \kear{} family intends to retain influential features from the model perspective, which results in a better preservation of the target object.
% TBD: add a table for the accuracies
The correctness and usefulness of the explanation are reflected by the higher retraining accuracies, especially compared to \roar{}.
Although both schemes highlight visually similar regions, the higher retraining accuracies confirm the better utility of positively influential features as determined by the explainer.
Another alternative to address the sign issue, \raftcAbs{}, follows the relevant-first manipulation, associating only with attribution amplitudes.
With all influential features occluded, regardless of which classes they contribute to, the target objects are masked out by replacement values and become unrecognizable.
The effective masking guided by the explanation method demonstrates the high quality of the derived explanations in both determining feature contribution directions (the sign of feature attribution) and isolating relevant features (high amplitudes) from irrelevant ones (low amplitudes).

Figures~\ref{apx:fig:resnet} and~\ref{apx:fig:swint} illustrate example explanations and corresponding manipulations in the large-scale settings.
Of particular note is Figure~\ref{apx:fig:swint}, which highlights the problematic attributions determined by IG. 
The listed examples from the ImageNet-SwinT test case demonstrate that IG can assign high attribution scores to uninformative regions (as pointed out by the arrows), helping to interpret the performance collapse presented in Figure~\ref{fig:ratioLarge}.
This observation is not limited to the shown examples and commonly happens across this specific test case.
A discussion regarding the potential causes of this failure has been made in Section~\ref{sec:exp_large}.

\begin{figure*}
    \centering
    \includegraphics[height=.93\textheight]{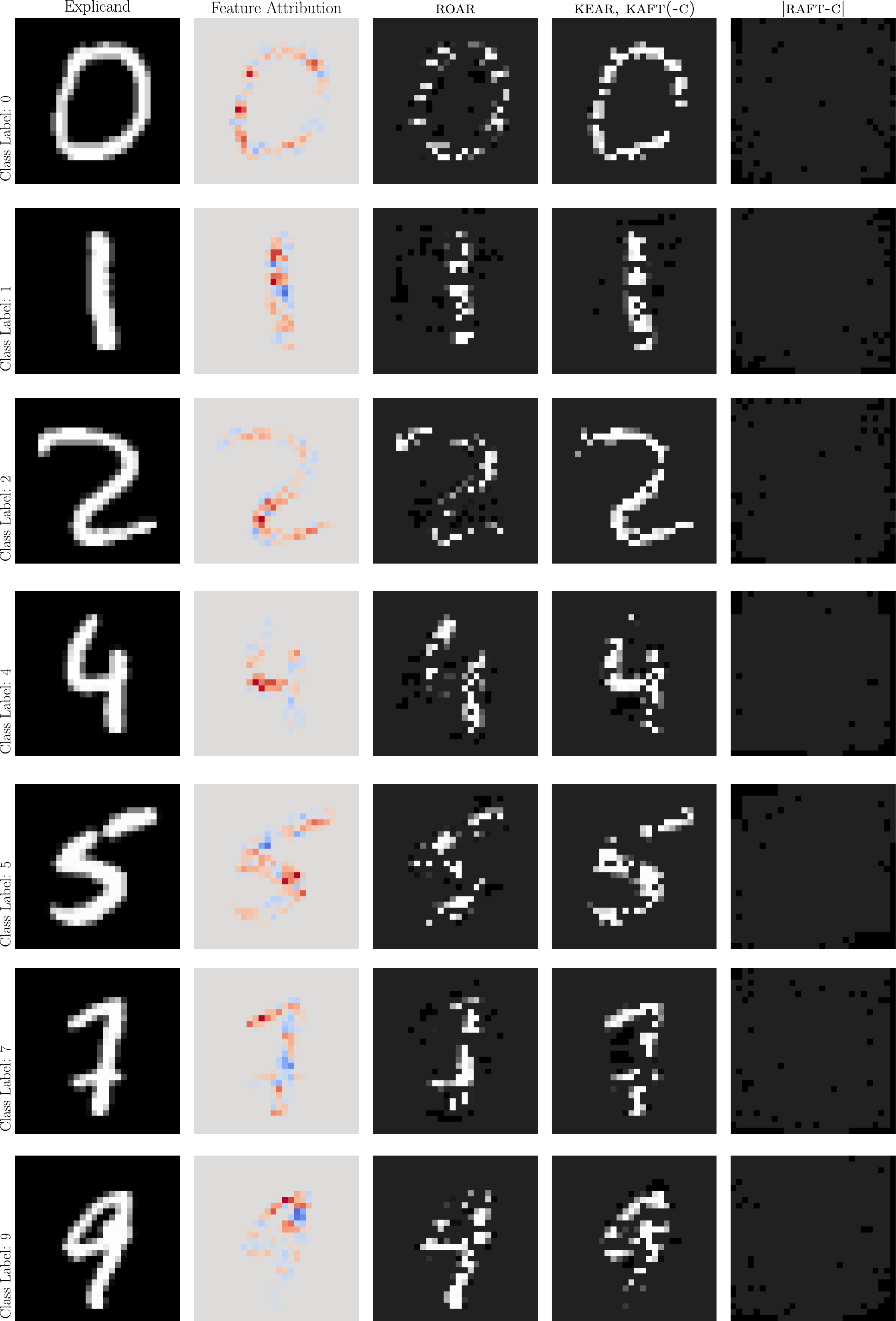}
    \caption{Examples of manipulated inputs from MNIST with $90\%$ of features occluded.}\label{apx:fig:mnist}
\end{figure*}

\begin{figure*}
    \centering
    \includegraphics[height=.93\textheight]{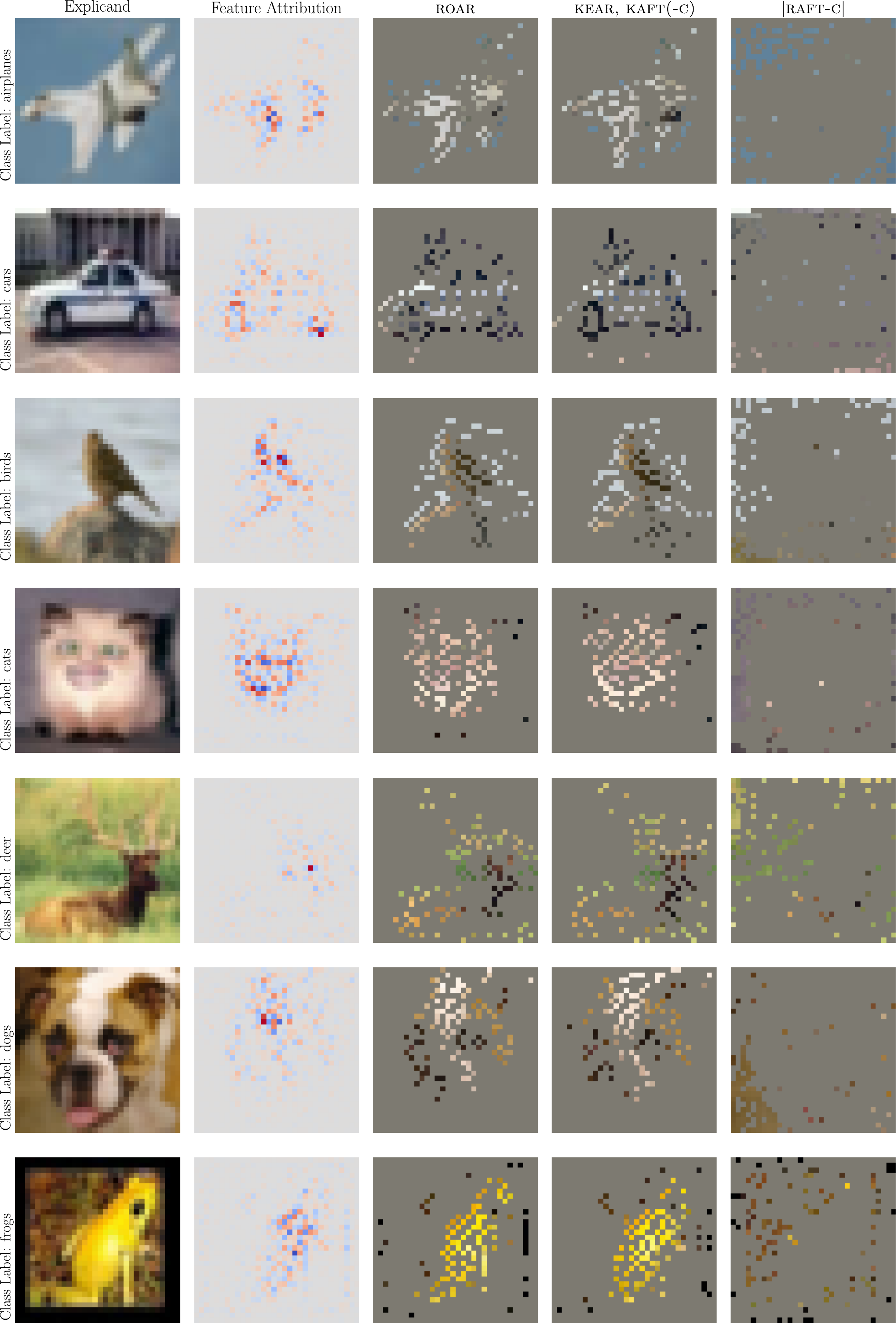}
    \caption{Examples of manipulated inputs from CIFAR10 with $90\%$ of features occluded.}\label{apx:fig:cifar}
\end{figure*}

\begin{figure*}
    \centering
    \includegraphics[height=.93\textheight]{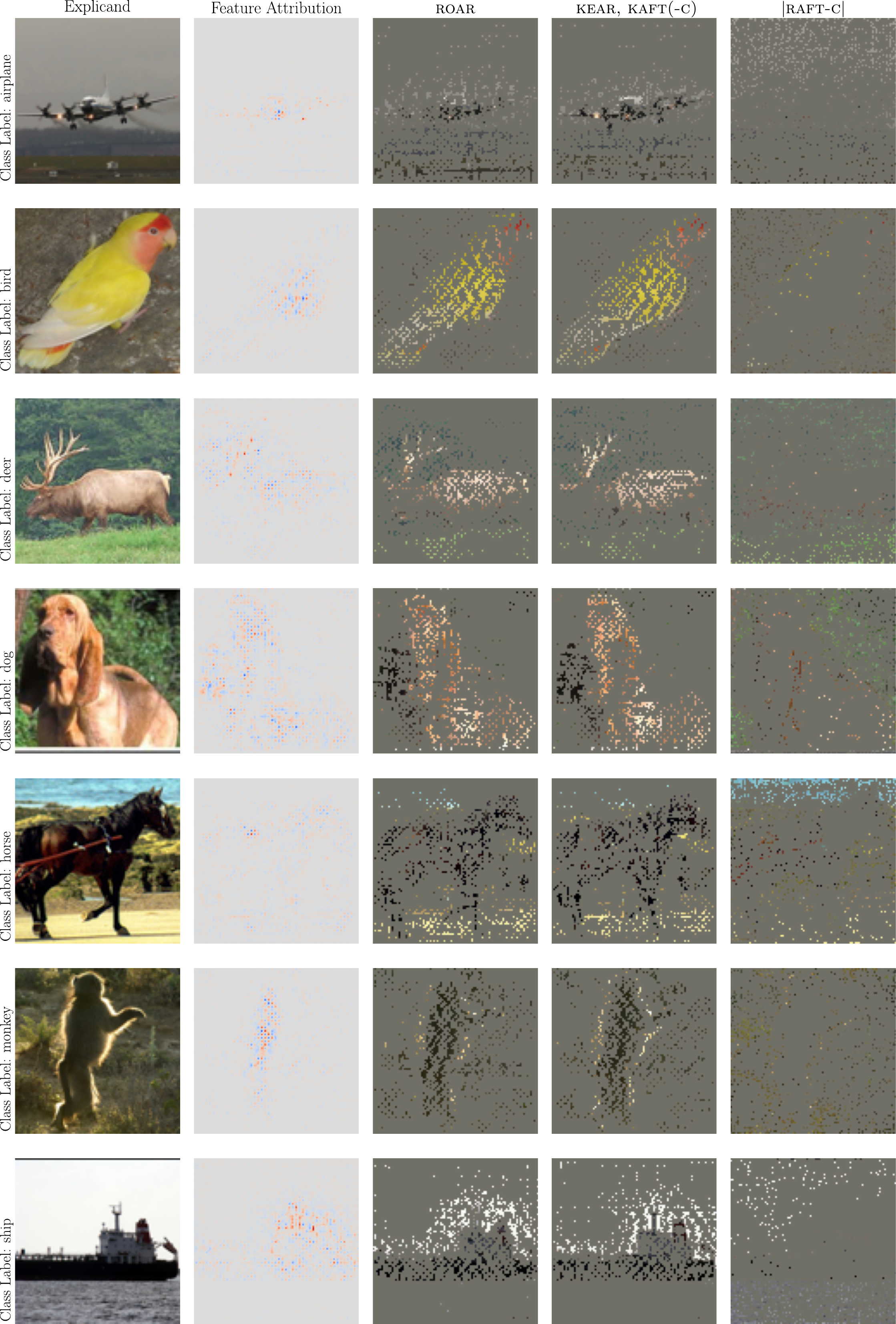}
    \caption{Examples of manipulated inputs from STL10 with $90\%$ of features occluded.}\label{apx:fig:stl}
\end{figure*}

\begin{figure*}
    \centering
    \includegraphics[height=.93\textheight]{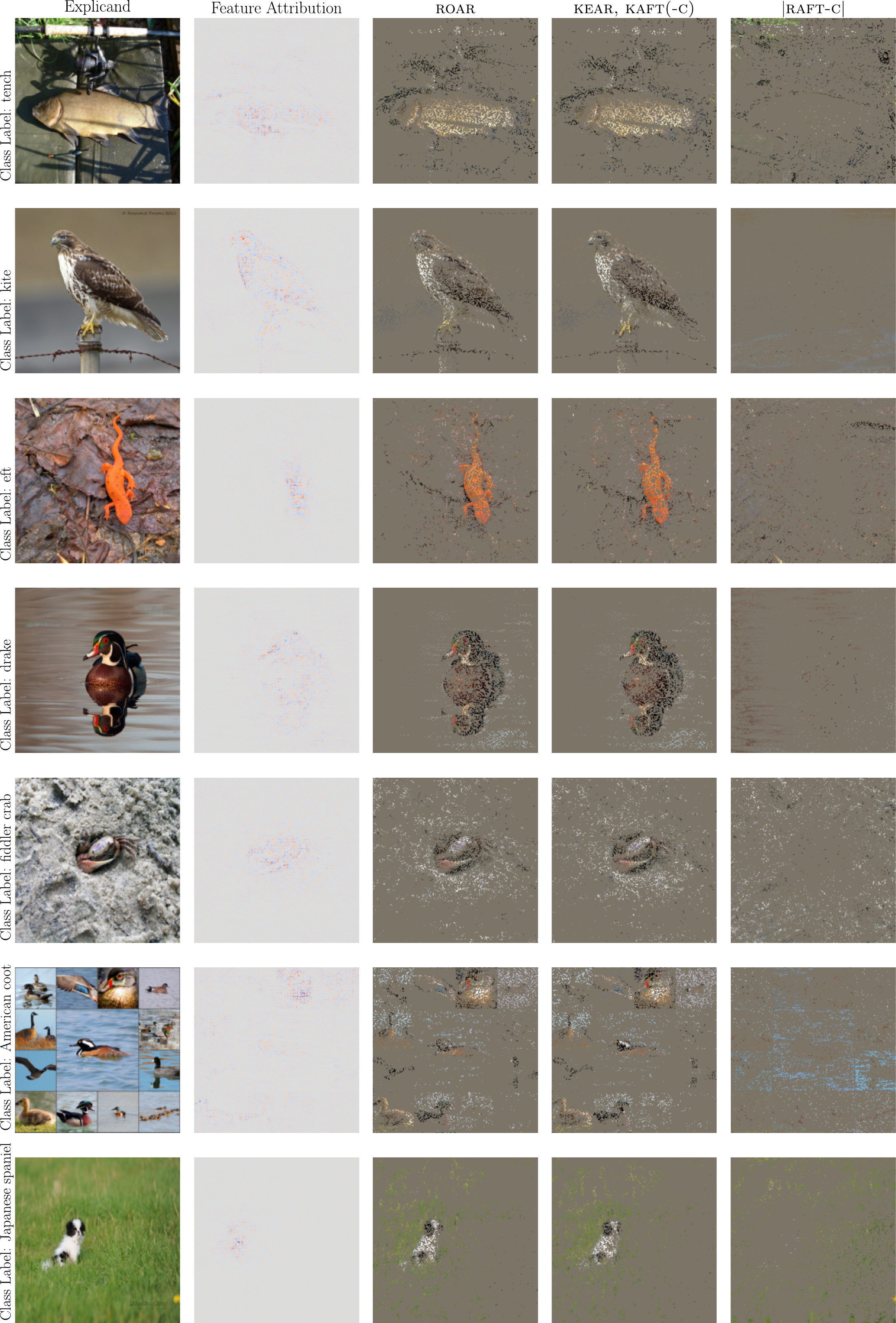}
    \caption{Examples of manipulated inputs from ImageNet-ResNet49 with $90\%$ of features occluded.}\label{apx:fig:resnet}
\end{figure*}

\begin{figure*}
    \centering
    \includegraphics[height=.93\textheight]{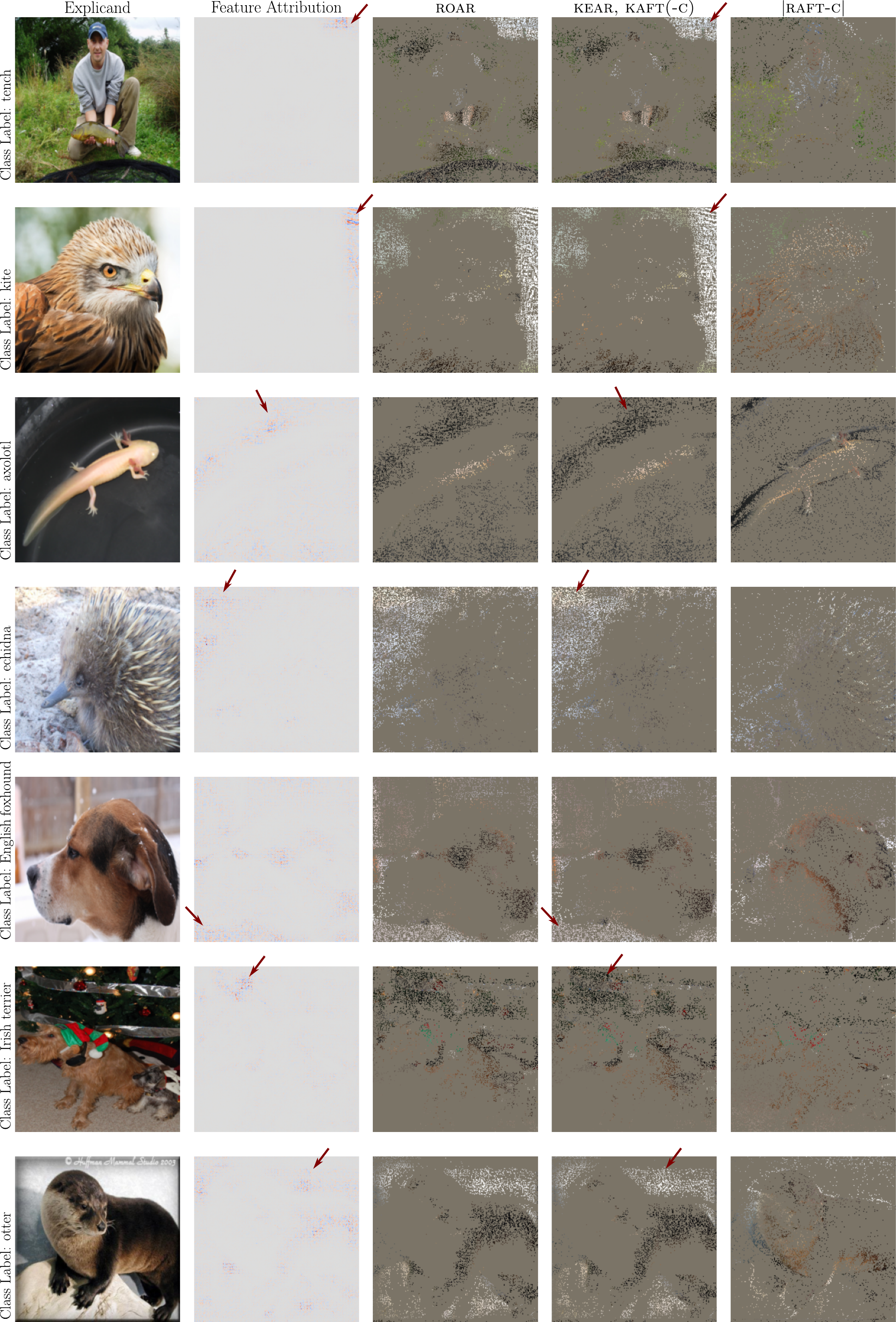}
    \caption{Examples of manipulated inputs from ImageNet-SwinT with $90\%$ of features occluded. The red arrow points to the problematic attributions in the second column, and how they are reflected during the evaluation process, as highlighted in the fourth column. }\label{apx:fig:swint}
\end{figure*}

\subsection{Limitation}
This paper emphasizes the importance of proper evaluation procedures within the retraining scheme framework, involving two key configurations: manipulation priority and model update constraints. 
With a clear intention to address the sign issue, the discussion centers on lowest-first removal combined with various model update protocols, which yield consistent results. 
The experiments also include \raftcAbs{} as a reference to demonstrate explainers' effectiveness in occluding features relevant to model decisions. 
However, a more extensive analysis of alternative combinations of evaluation settings has not been made due to space constraints. 

In fact, even though certain evaluation schemes may produce problematic assessments when considered in isolation, they may still offer valuable insights when interpreted in conjunction with complementary schemes. 
Concretely, \roar{} is prone to distortion due to negatively attributed features, which can result in unexpectedly high accuracies after retraining. 
However, when \roar{} and \kear{} are interpreted together, a meaningful comparison emerges: for an effective explainer, the retraining accuracy after lowest-first removal should exceed the result of highest-first removal. 
This intuition can be interpreted as: positively attributed features should support model decisions more directly than negative ones. 
Conversely, the opposite implies a potential misassignment of attribution signs, where negative features outweigh the positives in utility. 
As future work, we aim to expand this discussion by developing a more structured explanation evaluation framework that aligns results across retraining scheme variants with diverse manipulation priorities. 
By integrating outcomes from different configurations, such a framework could provide multifaceted assessments, contributing to a more objective and comprehensive understanding of explanation quality.

% \subsubsection{Additional Results from Extended Retraining Scheme Combinations}
% This paper mainly focuses on the \kear{} family for explanation evaluation.
% However, as discussed in Appendix~\ref{apx:abbrvs}, more choices are available by combining different retraining options, including manipulation priority and model update strategy.
% Table~\ref{apx:tab:full_small90p} exhaustively lists the results collected from different combinations of retraining schemes.
% For each test case, the retraining process is repeated for 5 times and the averaged performance is reported, followed by the standard deviation. 

%%%%%%%%%%%%%%%%%%%%%%%%%%%%%%%%%%%%%%%%%%%%%%%%%%%%%%%%%%%%
% PLACEHOLDER FOR CHECKLIST

\end{document}